\documentclass[journal]{IEEEtran}
\usepackage{amsmath,amssymb,amsfonts}
\usepackage{algorithmic}
\usepackage{graphicx}
\usepackage{textcomp}
\usepackage{xcolor}
\usepackage{times}
\usepackage{enumerate}
\usepackage{amsmath,amssymb,amsfonts}
\usepackage{mathtools} 
\usepackage{amsthm} 

\usepackage{cleveref}
\usepackage{mathrsfs}
\usepackage{comment}
\usepackage{cite}

\def\BibTeX{{\rm B\kern-.05em{\sc i\kern-.025em b}\kern-.08em
    T\kern-.1667em\lower.7ex\hbox{E}\kern-.125emX}}

\usepackage{placeins}
\usepackage{amsthm}
\newtheorem{definition}{Definition}

\newtheorem{theorem}{Theorem}
\newtheorem{corollary}{Corollary}
\newtheorem{assumption}{Assumption}
\newtheorem{remark}{Remark}

\newcommand{\Bern}{\operatorname{Bern}}
\newcommand{\ALG}{\mathrm{ALG}}
\newcommand{\TV}{\operatorname{TV}}
\newcommand{\KL}{\operatorname{KL}}
\newcommand{\Quantile}{\operatorname{Quantile}}
\newcommand{\WQuantile}{\operatorname{WQuantile}}

\ifCLASSINFOpdf
\else
\fi
\begin{document}
%
\title{Minimax Quantile Lower Bounds for Interactive Statistical Decision Making with Privacy}
%
%
%

\author{
\IEEEauthorblockN{Raghav Bongole, Amirreza Zamani, Tobias J. Oechtering, and Mikael Skoglund}\\
\IEEEauthorblockA{Department of Information Science and Engineering (ISE)\\
KTH Royal Institute of Technology\\
}
\thanks{This work is supported by the Knut and Alice Wallenberg Foundation.}
}

\maketitle




\def\x{{\mathbf x}}
\def\L{{\cal L}}


%
%

%

\begin{abstract}
Minimax risk and regret are expectation-based criteria and do not capture rare but consequential failures. To address this concern, we develop a \(\delta\)-explicit minimax-quantile theory for interactive statistical decision making (ISDM). We first provide structural relations between minimax quantiles, lower minimax quantiles, and 
minimax risk. This includes a quantile-to-expectation conversion and an equivalence between strict and lower minimax quantiles outside a countable set of confidence levels. We then derive two converse tools for ISDM: a high-probability interactive Fano's method and a high-probability interactive Le Cam's method. Then, we show that mutual-information (MI) privacy can be handled in the same framework by restricting the admissible decision class. For coordinatewise Gaussian privatization, we derive a two-point template that isolates the privacy-induced variance inflation. We instantiate this template for Gaussian mean estimation, and use the same two-point strategy directly for
two-armed Gaussian bandits. We then derive a minimax quantile lower bound for the \(K\)-armed Gaussian bandit problem, showing that the interactive Fano method captures the exploration cost over multiple possible best arms. The resulting lower bounds are explicit in the confidence level \(\delta\) and in the privacy budget for the private problems. They yield \(\log(1/\delta)/n\) scaling for squared-error Gaussian mean estimation, \(\sqrt{T\log(1/\delta)}\) scaling for two-armed bounded-mean Gaussian bandits, and \(\sqrt{KT\log(1/\delta)}\)-type scaling for the \(K\)-armed bandits, with privacy appearing through a Gaussian variance-inflation factor for the private problems.
\end{abstract}

\section{Introduction}
\label{sec:intro}

Minimax risk and minimax regret are central performance criteria in statistics, learning theory, bandits, and reinforcement learning. They characterize the smallest worst-case expected loss achievable by any algorithm against the hardest environment, and they underpin a large body of upper and lower bound theory \cite{ma2024high,audibert2009minimax,azar2017minimax,jin2018q,lattimore2020bandit,foster2021statistical}. In interactive settings, information-theoretic converse techniques such as Le Cam's and Fano's methods play a fundamental role in identifying the statistical difficulty of sequential decision problems \cite{lecam1973convergence,polyanskiy2014lecture,cover1999elements,yu1997assouad,chen2016bayes,duchi2013distance,chen2024assouad}. However, these classical minimax criteria are expectation-based and therefore insensitive to the tail behavior of the loss distribution. This limitation is especially important in safety-critical or reliability-sensitive settings. Two algorithms may have the same expected loss while exhibiting different high-confidence behavior. For example, in the bandit problem, two policies can have comparable expected regret even though one incurs substantially larger regret on a small but practically important fraction of runs. Such phenomena are not captured by expectation alone. Quantiles provide a natural tail-sensitive alternative by identifying the smallest (ideally non-trivial) threshold $r$ such that the left-tail event of the loss, \(\{L\le r\}\), is guaranteed with probability at least $1-\delta$ \cite{ma2024high}.

In non-interactive estimation, recent work by Ma et al.\ \cite{ma2024high} developed high-probability analogues of classical minimax lower-bound techniques and derived \(\delta\)-explicit minimax quantile lower bounds. At the same time, Chen et al.\ \cite{chen2024assouad} introduced the framework of interactive statistical decision making (ISDM), which unified passive estimation and interactive problems such as bandits and reinforcement learning, and systematized information-theoretic lower bounds under interaction. Yet these two lines of work leave an important gap: the former is non-interactive, while the latter focuses primarily on expectation-based minimax risk and weak-tail high-probability guarantees. A \(\delta\)-explicit lower-bound theory for \emph{strict} minimax quantiles in ISDM has remained largely open.

A second motivation comes from privacy-constrained decision making. In many modern statistical and sequential learning problems, the algorithm cannot act directly on raw data and instead operates through a privacy mechanism. Bounded mutual-information (MI) as a measure of privacy leakage is a natural information-theoretic notion in such settings. As discussed in information-theoretic privacy mechanism design
\cite{zamani2023privacy,zamani2022bounds}, bounded mutual information
provides an average leakage constraint that can be analytically convenient for
privacy-utility tradeoff analysis. However, because mutual information controls
average leakage rather than pointwise leakage, it does not prevent individual
realizations of the released data from revealing more information than the
average budget. 

From the minimax perspective, privacy alters the set of admissible decision rules and reduces the distinguishability of models. This raises a basic question: how do privacy constraints affect the minimax quantile of an interactive decision problem? Our perspective is that privacy based on bounded MI constraints can be treated as a constrained instance of ISDM in which the decision class is restricted to private algorithms.

In this paper, we develop a \(\delta\)-explicit minimax-quantile theory for ISDM and then specialize it to problems with bounded MI constraints. First, we prove that minimax quantiles under interaction admit the same converse tools that minimax risk admits in the expectation-based literature. We therefore derive high-probability interactive counterparts of Fano's and Le Cam's methods, establish structural links between strict minimax quantiles, lower minimax quantiles, and classical minimax risk. We then show how problems with bounded mutual-information privacy constraints, which we refer to as MI-private problems, fit naturally into the same framework through privacy-constrained admissible classes. For coordinatewise Gaussian privatization, this leads to a two-point template that can be used in concrete problems.

\subsection{Contributions}
 We summarize our contributions as follows:
\begin{itemize}
    \item  In Section~II, we discuss the most relevant prior work on high-probability minimax quantiles, information-theoretic lower bounds for interactive statistical decision making, and privacy-constrained minimax analysis. We position our work relative to existing minimax-quantile theories for non-interactive estimation and recent ISDM lower-bound frameworks.
    \item In Section~III, we introduce the ISDM framework together with \(\delta\)-explicit minimax-quantile criteria. We define strict minimax quantiles and lower minimax quantiles, discuss their relation to weak-tail formulations, and establish the notation and formal setup used throughout the paper.
    \item In Section~IV, we develop \(\delta\)-explicit lower-bound tools for ISDM. We first derive structural relations between minimax quantiles and classical minimax risk, including a quantile-to-expectation conversion and an equivalence between strict and lower minimax quantiles outside a countable set of confidence levels. We then develop two converse techniques for ISDM: a high-probability interactive Fano's method and a high-probability interactive Le Cam method, both yielding explicit minimax-quantile lower bounds.
    \item In Section~V, we show that statistical decision problems with bounded mutual-information privacy constraints can be formulated as ISDM problems with privacy-constrained admissible decision classes. We further specialize to coordinatewise Gaussian privatization and derive a two-point template that isolates the effect of coordinatewise Gaussian MI privacy through a variance-inflation factor.
    \item In Section VI, we instantiate the minimax quantile theory in five examples:
classical Gaussian mean estimation, classical two-armed Gaussian bandits, a
\(K\)-armed Gaussian bandit as an application of interactive Fano, and the
coordinatewise Gaussian MI-private counterparts of mean estimation and
two-armed bandits. We derive explicit \(\delta\)-dependent minimax-quantile lower bounds in each case and compare the private and classical settings, thereby making the privacy-induced variance inflation explicit.
    \item In Section~VII, we conclude by discussing the broader implications of the proposed framework and outlining several directions for future work, including extensions to other privacy notions, sharper multi-model lower bounds, and broader classes of interactive learning problems.
\end{itemize}

The conference version regarding this paper can be found in \cite{bongole2026risk}.
Beyond the specific technical results, the paper advocates a unifying perspective: tail-sensitive minimax lower bounds for estimation, interaction, and privacy can all be developed within a common ISDM framework.

\section{Related Work}
\label{sec:related}

Our work lies at the intersection of three research directions: minimax quantile lower bounds in non-interactive estimation, information-theoretic lower bounds for interactive statistical decision problems, and privacy-constrained statistical decision making.

\subsection{High-probability minimax quantiles in estimation}

Recent work by Ma et al.\ \cite{ma2024high} develops a theory of high-probability minimax lower bounds for estimation problems. Their results provide quantile analogues of classical techniques such as Fano's method and Le Cam's method, together with structural properties of minimax quantiles and connections back to expected minimax risk. In particular, they show how lower minimax quantiles can serve as a simple tool for proving strict minimax quantile lower bounds. This direction is highly relevant to the present paper, since it demonstrates that quantile-based converse theory is both feasible and useful.

However, the framework of \cite{ma2024high} is non-interactive. The observations are generated directly from a statistical model, and the lower-bound arguments do not need to account for adaptive data collection or sequential dependence induced by the algorithm. As a result, those tools do not immediately yield minimax quantile lower bounds for bandits, reinforcement learning, or more general interactive decision processes. One of our main goals is to extend this quantile viewpoint from estimation to ISDM.

\subsection{Information-theoretic lower bounds for interactive problems}

For interactive problems, Chen et al.\ \cite{chen2024assouad} introduce the ISDM framework, which unifies passive estimation and interactive decision making within a common formal framework. Within this framework, they develop interactive versions of Assouad, Fano's and Le Cam's methods, together with the decision-estimation coefficient (DEC) framework. These results provide a powerful and unifying theory for expectation-based minimax lower bounds under interaction.

The present work is closely aligned with Chen et al. \cite{chen2024assouad}, but differs in its objective and in the type of guarantees obtained. The results in \cite{chen2024assouad} are centered on minimax expected risk and on weak-tail high-probability statements. In contrast, our focus is on \emph{strict} minimax quantiles, defined through events of the form \(\{L>r\}\), and on lower bounds that are explicit in the risk level \(\delta\). This distinction is important because weak-tail bounds of the form \(\mathbb{P}(L\ge r)\ge \delta\) do not, in general, imply strict-tail quantile bounds without additional continuity assumptions.

The DEC literature also provides an important point of comparison. In the ISDM setting, \cite{chen2024assouad} formalize a quantile-DEC tailored to weak-tail guarantees and relate it to constrained DEC quantities. These results show that weak minimax quantiles can be controlled through problem-dependent complexity measures. Nevertheless, as stated there, they do not directly yield a \(\delta\)-explicit theory for strict minimax quantiles. Our contribution is complementary: instead of working through DEC, we develop direct \(\delta\)-explicit high-probability versions of interactive Fano's and Le Cam methods and use them to obtain strict minimax quantile lower bounds.

Classical bandit lower-bound results also include high-probability arguments. For example, the monograph of Lattimore and Szepesv\'ari \cite{lattimore2020bandit} develops weak-tail lower bounds for bandits, including two-armed Gaussian instances with the \(\sqrt{T\log(1/\delta)}\) scaling up to constants. These results are highly useful for specific problems, but they do not, by themselves, provide a generic minimax quantile framework for arbitrary interactive decision problems. Our aim is to extend such results to a more general ISDM theory.

\subsection{Privacy-constrained minimax analysis}

A third relevant line of work studies how privacy constraints affect statistical
and sequential decision problems. Information-theoretic privacy notions,
including mutual-information constraints, have been used to quantify privacy
leakage in data-release and privacy-utility tradeoff problems
\cite{zamani2023privacy,zamani2022bounds,sankar2013utility,calmon2015information,diaz2019robustness}. Such notions
naturally interact with minimax lower-bound methods because privacy mechanisms
reduce the distinguishability of the induced output laws. This connection is
central in locally private minimax estimation, where privacy constraints lead to
modified data-processing inequalities and private versions of Le Cam, Fano, and
Assouad methods \cite{duchi2013local,duchi2018minimax,duchi2019lower}.

Privacy constraints have also been studied in sequential learning and bandit
problems, mainly under differential-privacy or local-differential-privacy
requirements. Existing work derives regret upper and lower bounds for private
multi-armed bandits and related partial-information problems
\cite{tossou2016algorithms,basu2019differential,ren2020multi,azize2024concentrated, chen2025decision}.
Much of this literature focuses on expected error or expected regret, often in
non-interactive or estimation-centric settings, or on standard high-probability
regret guarantees under specific privacy notions. In contrast, our interest lies
in the minimax-quantile regime: we study how privacy constraints affect the
\((1-\delta)\)-quantile of the loss in an interactive decision problem.

We model privacy within the ISDM formulation by encoding the privacy constraint
as a restriction on the class of admissible decision rules. For the bounded MI
privacy constraint, this leads to a composed class of private decision rules
whose induced output laws can be analyzed using the same ISDM machinery as in
the standard case. This perspective separates the generic minimax-quantile
arguments from the privacy-specific verification of divergence and separation
conditions.

\subsection{Positioning of the present paper}

In summary, Ma et al.\ \cite{ma2024high} provide a quantile lower-bound theory without interaction, while Chen et al.\ \cite{chen2024assouad} provide an interaction-aware lower-bound theory centered on expectation and weak-tail guarantees. Our paper addresses the gap between these two directions by developing a \(\delta\)-explicit minimax quantile theory directly in ISDM. It then extends this theory to MI-private settings by viewing privacy as a restriction of the decision class. This yields a unified path from general interactive converse tools to concrete high-probability lower bounds for  estimation and bandit problems under privacy constraints.

\section{The interactive statistical decision making framework and Quantile Criteria}
\label{sec:framework}

We work within the interactive statistical decision making (ISDM) framework of \cite{chen2024assouad}, which accommodates both passive statistical estimation and interactive problems such as bandits and reinforcement learning. An ISDM instance is specified by a quadruple
\[
(\mathcal X,\mathcal M,\mathcal D,L),
\]
where \(\mathcal X\) is the outcome alphabet, \(\mathcal M\) is the model class, \(\mathcal D\) is the class of admissible decision rules, and
\[
L:\mathcal M\times \mathcal X \to [0,\infty)
\]
is the loss function.

For each model \(M\in\mathcal M\) and each decision rule \(\ALG\in\mathcal D\), the interaction between the environment and the decision maker induces a probability law
\[
\mathbb P^{M,\ALG}\in \Delta(\mathcal X)
\]
over the final outcome \(X\in\mathcal X\). In a passive estimation problem, \(X\) may be the observed sample together with an estimator; in an interactive problem, \(X\) may represent the full transcript of actions and observations. The loss \(L(M,X)\) is evaluated under the induced law \(\mathbb P^{M,\ALG}\), and all probabilities and expectations in this paper are taken with respect to this law unless otherwise stated.

\subsection{Notation}

For convenience, we use the following notation throughout the paper. We write
\(\TV(\cdot,\cdot)\) for total variation distance, \(\KL(\cdot\|\cdot)\) for
Kullback--Leibler divergence, and \(I(\cdot;\cdot)\) for mutual information.

For a fixed model \(M\in\mathcal M\) and decision rule
\(\ALG\in\mathcal D\), we write
\[
P^{M,\ALG}\in\Delta(\mathcal X)
\]
for the law of the final outcome \(X\) induced by the interaction between the
environment \(M\) and the decision rule \(\ALG\). Thus, probabilities and
expectations of the form
\[
\mathbb P^{M,\ALG}(\cdot),
\qquad
\mathbb E^{M,\ALG}[\cdot],
\]
are taken with respect to \(X\sim P^{M,\ALG}\). Equivalently,
\(P^{M,\ALG}\) may be viewed as the conditional law of the outcome given the
model \(M\) and the decision rule \(\ALG\).

For a prior \(\mu\in\Delta(\mathcal M)\), expressions such as
\[
\mathbb P_{M\sim\mu,\;X\sim P^{M,\ALG}}(\cdot),
\qquad
\mathbb E_{M\sim\mu,\;X\sim P^{M,\ALG}}[\cdot],
\]
mean that \(M\) is first drawn from \(\mu\), and then, conditionally on
\(M\), the outcome \(X\) is drawn according to \(P^{M,\ALG}\). Similarly,
\[
\mathbb P_{M\sim\mu,\;X\sim Q}(\cdot)
\]
means that \(M\sim\mu\) and \(X\sim Q\), with \(X\) drawn according to the reference law \(Q\).

We use \(\Bern(p)\) for the Bernoulli distribution with parameter \(p\). We use
\(x_n\downarrow x\), respectively \(x_n\uparrow x\), to denote convergence to
\(x\) from above, respectively from below.

\subsection{Minimax risk and minimax quantiles}

The classical benchmark is the minimax risk \cite{polyanskiy2014lecture}.

\begin{definition}[Minimax risk]
The minimax risk is defined as
\[
\mathfrak M
:=
\inf_{\ALG\in\mathcal D}\;
\sup_{M\in\mathcal M}\;
\mathbb E^{M,\ALG}[L(M,X)].
\]
\end{definition}

While \(\mathfrak M\) captures worst-case performance in expectation, it does not distinguish between algorithms with similar expected loss but substantially different tail behavior. To study high-confidence performance, we use minimax quantiles.

\begin{definition}[Quantile]
For any \(\delta\in(0,1]\), the \((1-\delta)\)-quantile of the loss under \((M,\ALG)\) is
\[
\Quantile(1-\delta,\mathbb P^{M,\ALG},L)
:=
\inf\Bigl\{
r\in[0,\infty]:
\mathbb P^{M,\ALG}\bigl(L(M,X)>r\bigr)\le \delta
\Bigr\}.
\]
\end{definition}

\begin{definition}[Strict minimax quantile]
For any \(\delta\in(0,1]\), the strict minimax quantile is
\[
\mathfrak M(\delta)
:=
\inf_{\ALG\in\mathcal D}\;
\sup_{M\in\mathcal M}\;
\Quantile(1-\delta,\mathbb P^{M,\ALG},L).
\]
\end{definition}

Following \cite{ma2024high}, we also consider the lower minimax quantile, which is often easier to lower bound directly.

\begin{definition}[Lower minimax quantile]
For \(\delta\in(0,1]\), the lower minimax quantile is
\[
\mathfrak M_{-}(\delta)
:=
\inf\Bigl\{
r\in[0,\infty]:
\inf_{\ALG\in\mathcal D}\;
\sup_{M\in\mathcal M}\;
\mathbb P^{M,\ALG}\bigl(L(M,X)>r\bigr)\le \delta
\Bigr\}.
\]
\end{definition}

The quantity \(\mathfrak M(\delta)\) finds the smallest threshold that an algorithm can guarantee with confidence \(1-\delta\), over the model class, while \(\mathfrak M_{-}(\delta)\) captures the tail-probability formulation that will be more convenient for our converse results.

\subsection{Strict and weak tail conventions}

Since several prior high-probability lower bounds in interactive learning are expressed using weak-tail events of the form \(\{L\ge r\}\), we make our convention explicit. Throughout the paper, our primary object is the \emph{strict} quantile defined through the event \(\{L>r\}\). This choice is deliberate as it yields a correspondence with the lower minimax quantiles and leads to lower bounds on the entire strict quantile curve \(\delta\mapsto \mathfrak M(\delta)\).

For comparison, one may also define the weak quantile by
\[
\WQuantile(1-\delta,\mathbb P^{M,\ALG},L)
:=
\inf\Bigl\{
r\in[0,\infty]:
\mathbb P^{M,\ALG}\bigl(L(M,X)\ge r\bigr)\le \delta
\Bigr\},
\]
and the corresponding minimax weak quantile
\[
\mathfrak M_W(\delta)
:=
\inf_{\ALG\in\mathcal D}\;
\sup_{M\in\mathcal M}\;
\WQuantile(1-\delta,\mathbb P^{M,\ALG},L).
\]
In general, weak-tail lower bounds do not imply strict-tail lower bounds without additional regularity assumptions, for example when \(\mathbb P(L=r)>0\). For this reason, our main results are stated directly for strict minimax quantiles and lower minimax quantiles.

The next section develops \(\delta\)-explicit lower-bound tools for \(\mathfrak M(\delta)\) and \(\mathfrak M_{-}(\delta)\) in ISDM. Those results will later be specialized to privacy-constrained admissible classes, including MI-private decision rules.

\section{Minimax Quantile Lower Bounds in ISDM}
\label{sec:main-results}

This section develops \(\delta\)-explicit lower-bound tools for minimax quantiles in ISDM. We proceed in three steps. First, we connect minimax quantiles to the classical minimax risk and relate strict minimax quantiles to lower minimax quantiles. Second, we develop a high-probability interactive Fano's method. Third, we develop a high-probability interactive Le Cam's method. Together, these results provide a converse toolkit for proving tail-sensitive lower bounds under interaction.

\subsection{Bridges between expectation, tails, and quantiles}

We first show that lower bounds on minimax quantiles imply corresponding
lower bounds on the minimax risk. In this sense, the quantile-based
criterion provides a tail-sensitive refinement of the expectation-based minimax
criterion.

\begin{theorem}[Quantile-to-expectation conversion in ISDM]
\label{thm:quantile-to-expectation}
For every \(\delta\in(0,1]\),
\[
\mathfrak M
=
\inf_{\ALG\in\mathcal D}\;
\sup_{M\in\mathcal M}\;
\mathbb E^{M,\ALG}[L(M,X)]
\;\ge\;
\delta\,\mathfrak M(\delta).
\]
\end{theorem}

\begin{proof}[Proof sketch] The proof adapts the argument of \cite[Proposition~2]{ma2024high} to the ISDM
setting by working with the induced laws \(\mathbb{P}^{M,\ALG}\) associated with
interactive decision rules. The full details are given in the appendix.
\end{proof}

The next result shows that lower minimax quantiles are essentially equivalent to strict minimax quantiles. This is important because lower minimax quantiles are often easier to control directly through tail-probability arguments.

\begin{theorem}[Lower minimax quantile relation in ISDM]
\label{thm:lower-minimax-quantile-relation}
For every \(\delta\in(0,1]\) and every \(\xi\in(0,\delta)\),
\[
\mathfrak M_{-}(\delta)
\;\le\;
\mathfrak M(\delta)
\;\le\;
\mathfrak M_{-}(\delta-\xi).
\]
Consequently,
\[
\mathfrak M(\delta)=\mathfrak M_{-}(\delta)
\]
for all \(\delta\in(0,1]\) except a countable set. Moreover, if for some \(r\ge 0\),
\[
\inf_{\ALG\in\mathcal D}\;
\sup_{M\in\mathcal M}\;
\mathbb P^{M,\ALG}\bigl(L(M,X)>r\bigr)
>
\delta,
\]
then
\[
\mathfrak M(\delta)\ge \mathfrak M_{-}(\delta)\ge r.
\]
\end{theorem}

\begin{proof}[Proof sketch] The proof adapts the argument of \cite[Theorem 4]{ma2024high} to the ISDM
setting by working with the induced laws \(\mathbb{P}^{M,\ALG}\) associated with
interactive decision rules. For the full proof, we refer the reader to the appendix.
\end{proof}

\begin{remark}
    Theorem~\ref{thm:lower-minimax-quantile-relation} provides a convenient
route for deriving strict minimax-quantile lower bounds. In many lower-bound arguments one first proves a statement of the form
\[
\inf_{\ALG\in\mathcal D}\sup_{M\in\mathcal M}\mathbb P^{M,\ALG}(L(M,X)>r)>\delta,
\]
and then converts it into a strict minimax quantile lower bound using the last part of the theorem.
\end{remark}

\subsection{A high-probability interactive Fano's method}

We now derive an interactive Fano-type lower bound for lower minimax quantiles. The result is stated for general \(f\)-divergences and yields a \(\delta\)-explicit lower bound that is uniform over algorithms.

\begin{theorem}[High-probability interactive Fano]
\label{thm:hp-interactive-fano}
Fix an \(f\)-divergence \(D_f\), a prior \(\mu\in\Delta(\mathcal M)\), and a
threshold \(\Delta>0\). For each \(\ALG\in\mathcal D\), let
\(Q_\ALG\in\Delta(\mathcal X)\) be a reference distribution. Define
\[
\bar\rho_{\Delta,Q_\ALG}
:=
\mathbb P_{M\sim\mu,\;X\sim Q_\ALG}
\bigl(L(M,X)\le \Delta\bigr),
\]
and
\[
d_{f,\epsilon}(p)
:=
\begin{cases}
D_f\bigl(\Bern(1-\epsilon)\,\|\,\Bern(p)\bigr),
& p\le 1-\epsilon,\\[0.75ex]
0, & p>1-\epsilon.
\end{cases}
\]
Let
\[
\epsilon^\star
:=
\sup_{\substack{\{Q_\ALG\}_{\ALG\in\mathcal D}\\ \epsilon\in[0,1]}}
\left\{
\epsilon:
\sup_{\ALG\in\mathcal D}
\left[
\mathbb E_{M\sim\mu}
\bigl[
D_f(\mathbb P^{M,\ALG}\|Q_\ALG)
\bigr]
-
d_{f,\epsilon}(\bar\rho_{\Delta,Q_\ALG})
\right]
<0
\right\}.
\]
Then, for every \(\delta\in[0,\epsilon^\star)\), we have
\[
\mathfrak M_{-}(\delta)\ge \Delta.
\]
\end{theorem}

\begin{proof}[Proof sketch]
The proof adapts the interactive Fano method of
\cite[Thm.~2]{chen2024assouad} to the success event
\[
\{L(M,X)\le \Delta\}
\]
rather than the event \(\{L(M,X)<\Delta\}\). For any fixed algorithm
\(\ALG\), reference law \(Q_\ALG\), and level \(\epsilon\), the Fano argument
implies that if
\[
\mathbb E_{M\sim\mu}
\bigl[
D_f(\mathbb P^{M,\ALG}\|Q_\ALG)
\bigr]
<
d_{f,\epsilon}(\bar\rho_{\Delta,Q_\ALG}),
\]
then
\[
\mathbb P_{M\sim\mu,\;X\sim \mathbb P^{M,\ALG}}
\bigl(L(M,X)>\Delta\bigr)
\ge \epsilon.
\]
By the definition of \(\epsilon^\star\), for every
\(\epsilon<\epsilon^\star\) one can choose a reference family
\(\{Q_\ALG\}_{\ALG\in\mathcal D}\) such that this condition holds uniformly
over all algorithms. Hence,
\[
\inf_{\ALG\in\mathcal D}
\mathbb P_{M\sim\mu,\;X\sim \mathbb P^{M,\ALG}}
\bigl(L(M,X)>\Delta\bigr)
\ge \epsilon.
\]
Since the outer probability averages over the prior \(\mu\), this implies
\[
\inf_{\ALG\in\mathcal D}
\sup_{M\in\mathcal M}
\mathbb P^{M,\ALG}(L(M,X)>\Delta)
\ge \epsilon.
\]
The conclusion \(\mathfrak M_{-}(\delta)\ge \Delta\) for all
\(\delta<\epsilon^\star\) follows from
Theorem~\ref{thm:lower-minimax-quantile-relation}.
\end{proof}

\begin{remark}
 Theorem~\ref{thm:hp-interactive-fano} yields a lower-bound curve in \(\delta\), rather than a bound at a single confidence level. It also proposes a \emph{strict-tail} statement involving \(\mathbb P(L>\Delta)\), which is the form needed for our minimax quantile analysis.
\end{remark}

\subsection{A high-probability interactive Le Cam's method}

We next derive a two-point lower-bound method for lower minimax quantiles. 

\begin{theorem}[High-probability interactive Le Cam]
\label{thm:hp-lecam}
Let \(\delta\in(0,\tfrac12)\), and let \(M_1,M_2\in\mathcal M\). Suppose that the loss satisfies the uniform separation condition
\[
L(M_1,x)+L(M_2,x)\ge 2\Delta
\qquad
\text{for all }x\in\mathcal X.
\]
Then the following hold.
\begin{enumerate}[(a)]
    \item If 
    $\sup_{\ALG\in\mathcal D}
    \TV(\mathbb P^{M_1,\ALG},\mathbb P^{M_2,\ALG})<1-2\delta,
    $
    then
    \[
    \mathfrak M_{-}(\delta)\ge \Delta.
    \]

    \item If 
    $\sup_{\ALG\in\mathcal D}
    \KL(\mathbb P^{M_1,\ALG}\|\mathbb P^{M_2,\ALG})
    <
    \log\!\left(\frac{1}{4\delta(1-\delta)}\right),
    $
    then
    \[
    \mathfrak M_{-}(\delta)\ge \Delta.
    \]
\end{enumerate}
\end{theorem}
\begin{proof}[Proof] For the full proof, we refer the reader to the appendix.
\end{proof}

\begin{remark}
    Theorem~\ref{thm:hp-lecam} will be particularly useful later for applications. In MI-private problems, privacy mechanisms reduce distinguishability between models, and this often yields an upper bound on
\(
\KL(\mathbb P^{M_1,\ALG}\|\mathbb P^{M_2,\ALG})
\)
or
\(
\TV(\mathbb P^{M_1,\ALG},\mathbb P^{M_2,\ALG}),
\)
while the loss satisfies a two-point separation condition.
\end{remark}

\subsection{Discussion}

Theorems~\ref{thm:quantile-to-expectation}, \ref{thm:lower-minimax-quantile-relation}, \ref{thm:hp-interactive-fano}, and \ref{thm:hp-lecam} together provide a converse toolkit for minimax quantiles in ISDM. The quantile-to-expectation conversion shows that quantile lower bounds imply expectation lower bounds. The lower-minimax-quantile relation shows that tail lower bounds suffice to control strict minimax quantiles. The interactive Fano method provides prior-based lower bounds that are well
suited to capture hardness arising from large model classes. In contrast, the
interactive Le Cam method offers a two-point reduction that is often more direct
to verify in concrete applications.

In the next section, we show that this theory transfers directly to privacy-constrained interactive decision problems once privacy is modeled as a restriction of the admissible decision class.

\section{MI Privacy and Gaussian-Privatization}
\label{sec:mi-private}

We now specialize the ISDM quantile theory to statistical decision problems
subject to bounded mutual-information (MI) privacy constraints. The privacy
setting considered in this section is naturally interpreted as involving two
distinct roles. The first role is a data holder, or privacy mechanism, that has
access to the raw sensitive data. The second role is a learner or decision
maker that does not observe the raw data directly, but only receives a
privatized output generated from it. This separation is common in applications where useful statistical decisions
must be made from sensitive data. For example, a hospital may have blood-test
results from a group of patients and may want an external research team to
estimate the average effect of a new medicine. The hospital may not release
each patient's raw test result directly. Instead, it can add random noise to the
test results and release only the privatized measurements. The research team
then computes an estimate using only these privatized measurements. In this
case, the raw test-result vector is the sensitive object, while the privatized
measurements and the resulting estimate form the released output. A similar separation appears in interactive decision problems.
In a privacy-preserving implementation, the raw data may not be revealed
directly to the learning policy. Instead, the platform or data holder first
passes the response through a privacy mechanism, and the learner observes only the
privatized data. The learner then uses the privatized history to decide
which action to play in future rounds. Thus, the raw data  sequence can be the sensitive object, while the released transcript consists of the privatized responses.

The privacy mechanism may sit between the data-generating
process and the learner. In non-interactive estimation \cite{ma2024high}, the learner receives a privatized
sample instead of the raw sample. In interactive problems such as bandits \cite{lattimore2020bandit}, the learner receives privatized
feedback after each action instead of the raw reward. Thus, the learner's
admissible decision rules are constrained by the information made available
through the privacy mechanism.

In the MI privacy formulation \cite{lopuhaa2020privacy,zamani2023privacy,zamani2022bounds,rassouli2021perfect,asoodeh2016information, basciftci2016privacy, shkel2020secrecy,du2017principal,zamani2024information,zamani2025variable,sankar2013utility, makhdoumi2014information}, information leakage is measured by the mutual
information between the sensitive object \(S\) and the released output \(Y\).
The constraint
\[
I(S;Y)\le \varepsilon
\]
requires that, on average, the released output \(Y\) reduces uncertainty about
the sensitive object \(S\) by at most \(\varepsilon\). Smaller values of
\(\varepsilon\) therefore correspond to stronger privacy requirements, while
larger values permit more information about the sensitive object to be made
available for learning.

Equivalently, one may view the privacy constraint through an adversarial
interpretation. An adversary observes the released output \(Y\) and attempts to
infer information about the sensitive object \(S\). The adversary may know the
model class, the learning algorithm, and the privacy mechanism, but does not
directly observe the raw data. In the estimation example, the adversary may try
to infer patient-level test results from the privatized measurements or from
the released estimate. In the advertising example, the adversary may try to
infer a user's private interests or preferences from the sequence of displayed
advertisements and privatized responses. The role of the privacy mechanism is
to ensure that this information gain, quantified by mutual information, remains
bounded.

This MI criterion controls average information leakage. It is therefore weaker
than pointwise privacy notions such as differential privacy, which impose
worst-case guarantees under neighboring datasets \cite{lopuhaa2020privacy}. Nevertheless, MI privacy is a
natural information-theoretic privacy constraint for minimax analysis. It
directly limits how distinguishable sensitive data-generating mechanisms can
become after privatization, and this distinguishability is precisely the
quantity that appears in information-theoretic lower bounds. 

From the perspective of ISDM, an MI privacy constraint changes the admissible
decision rules. A private estimator is constrained to act on privatized samples
rather than on raw samples. A private bandit policy is constrained to act on
privatized feedback histories rather than on raw reward histories. Therefore,
privacy affects performance in two related ways: it restricts the set of
feasible decision rules, and it changes the induced distribution of the output
on which the loss is evaluated. The key point of this section is that such
privacy-constrained problems do not require a new minimax framework. They can be
treated as ISDM instances with a constrained admissible decision class. Once this
constrained class is identified, the general lower-bound results of
Section~\ref{sec:main-results} apply directly. The framework in our work can also be extended to other privacy measures.

\subsection{MI Privacy as a privacy-constrained admissible class}

Let \(\mathcal M\) be a model class, let \(\mathcal D\) be a class of base
interactive decision rules, and let
\[
(S,X)\sim P^{M,\ALG}_{S,X}\in\Delta(\mathcal S\times\mathcal X)
\]
denote the joint law of a sensitive object \(S\) and a pre-release transcript
\(X\) under model \(M\in\mathcal M\) and algorithm \(\ALG\in\mathcal D\).
Here, \(S\) represents the object to be protected, whereas \(X\) denotes the
information available before privatization. Depending on the problem, \(X\) may
be a sample, an estimator based on the sample, a collection of rewards, or a
full interactive transcript. Here, we assume \(S\) and \(X\) are arbitrarily
correlated and drawn from the joint distribution \(P^{M,\ALG}_{S,X}\).

A privacy mechanism is a Markov kernel
\[
Q:\mathcal X\to\Delta(\mathcal Y),
\qquad
Y\,|\,X\sim Q(\cdot\,|\,X),
\]
which maps the pre-release transcript \(X\) to a released output \(Y\). This
mechanism induces a released-output law
\(P_Y^{M,Q,\ALG}\in\Delta(\mathcal Y)\). For a privacy budget
\(\varepsilon>0\), we say that \(Q\) is \(\varepsilon\)-MI-feasible for
\(\ALG\) if
\[
\sup_{M\in\mathcal M} I_{P^{M,Q,\ALG}}(S;Y)\le \varepsilon.
\]
The supremum over \(M\in\mathcal M\) requires the privacy guarantee to hold
uniformly over the model class. Thus, the same mechanism must control the
information leakage about \(S\) under every admissible data-generating model.

Let \(\mathcal Q\) be a prescribed class of privacy mechanisms, and define
\[
\mathcal Q_\varepsilon(\ALG)
:=
\Bigl\{
Q\in\mathcal Q:
\sup_{M\in\mathcal M} I_{P^{M,Q,\ALG}}(S;Y)\le \varepsilon
\Bigr\}.
\]
Given a loss \(L:\mathcal M\times\mathcal Y\to[0,\infty)\), the MI-private
minimax quantile can be written as
\begin{align}
\mathfrak M(\delta;\varepsilon)
:={}&
\inf_{\ALG\in\mathcal D}\;
\inf_{Q\in\mathcal Q_\varepsilon(\ALG)}
\sup_{M\in\mathcal M}
\Quantile \bigl(1-\delta,P_Y^{M,Q,\ALG},L\bigr).
\label{eq:mi-private-two-stage}
\end{align}
This expression reflects the two-stage description of the private decision
problem: one specifies a base decision rule \(\ALG\) and an admissible privacy
mechanism \(Q\) satisfying the MI constraint. The resulting performance is
evaluated only through the released output \(Y\), as this is the information
available after privatization.

It is convenient to encode the privacy mechanism and the downstream decision
rule into a single effective private algorithm. We write \(\ALG^{\mathrm{priv}}\) for the
overall procedure that induces the released output \(Y\). This notation is
understood extensionally: \(\ALG^{\mathrm{priv}}\) denotes any procedure, possibly obtained by
applying a privacy mechanism before, after, or during the execution of a base
decision rule, whose released output satisfies the MI constraint. Define the
privacy-constrained admissible class
\[
\mathcal D_\varepsilon^{\mathrm{priv}}
:=
\Bigl\{
\ALG^{\mathrm{priv}} :
\ALG^{\mathrm{priv}} \text{ induces a released output } Y\in\mathcal Y
\text{ and }
\sup_{M\in\mathcal M}
I_{P^{M,\ALG^{\mathrm{priv}}}}(S;Y)\le \varepsilon
\Bigr\}.
\]
For \(\ALG^{\mathrm{priv}}\in\mathcal D_\varepsilon^{\mathrm{priv}}\), let
\(P^{M,\ALG^{\mathrm{priv}}}\in\Delta(\mathcal Y)\) denote the induced law of the released
output \(Y\). Then, we have
\begin{equation}
\mathfrak M(\delta;\varepsilon)
:=
\inf_{\ALG^{\mathrm{priv}}\in\mathcal D_\varepsilon^{\mathrm{priv}}}
\sup_{M\in\mathcal M}
\Quantile(1-\delta,P^{M,\ALG^{\mathrm{priv}}},L),
\label{eq:mi-private-composed}
\end{equation}
and the private lower minimax quantile is
\begin{align}
\mathfrak M_-(\delta;\varepsilon)
:={}&
\inf \Bigl\{
r\ge 0:
\inf_{\ALG^{\mathrm{priv}}\in\mathcal D_\varepsilon^{\mathrm{priv}}}
\sup_{M\in\mathcal M}
P^{M,\ALG^{\mathrm{priv}}}(L(M,Y)>r)\le \delta
\Bigr\}.
\label{eq:mi-private-lower}
\end{align}

Thus, MI-private minimax quantiles are exactly minimax quantiles in an ISDM
instance with outcome space \(\mathcal Y\), loss \(L(M,Y)\), and
privacy-constrained admissible class
\(\mathcal D_\varepsilon^{\mathrm{priv}}\). This class need not be a literal
subset of the original non-private class \(\mathcal D\), since privatization
may change the information available to the learner and may also change the
output space. The privacy constraint is therefore encoded through the
admissible private procedures for the privatized ISDM instance.

\begin{remark}[Relation to LDP in finite-alphabet settings]
In finite-alphabet settings, MI privacy can be compared directly with local
differential privacy. Let
\[
\mathcal D_{\varepsilon}^{\mathrm{LDP}}
:=
\Bigl\{
\ALG^{\mathrm{priv}}:
\ALG^{\mathrm{priv}} \text{ induces } Y \text{ through a channel } Q
\text{ such that }
Q(y\mid s)\le e^\varepsilon Q(y\mid s')
\ \forall\, y\in\mathcal Y,\ s,s'\in\mathcal S
\Bigr\},
\]
and let
\[
\mathcal D_{\varepsilon}^{\mathrm{MI}}
:=
\Bigl\{
\ALG^{\mathrm{priv}}:
\sup_{M\in\mathcal M}
I_{P^{M,\ALG^{\mathrm{priv}}}}(S;Y)
\le \varepsilon
\Bigr\}.
\]
Under the finite-alphabet assumptions of
\cite[Lemma~1]{lopuhaa2020privacy}, every \(\varepsilon\)-LDP channel satisfies
\(I(S;Y)\le\varepsilon\). Hence
\[
\mathcal D_{\varepsilon}^{\mathrm{LDP}}
\subseteq
\mathcal D_{\varepsilon}^{\mathrm{MI}}.
\]
Consequently,
\[
\mathfrak M_{\mathrm{MI}}(\delta;\varepsilon)
\le
\mathfrak M_{\mathrm{LDP}}(\delta;\varepsilon).
\]
Thus, any MI-private minimax-quantile lower bound at level \(\varepsilon\)
also applies to the corresponding finite-alphabet
\(\varepsilon\)-LDP-constrained problem.
\end{remark}

All lower-bound results from Section~\ref{sec:main-results}
apply directly to the MI-private problem by replacing the admissible class
\(\mathcal D\) with \(\mathcal D_\varepsilon^{\mathrm{priv}}\), the outcome
\(X\) with the released output \(Y\), and the induced laws \(P^{M,\ALG}\) with
\(P^{M,\ALG^{\mathrm{priv}}}\). In particular, the quantile-to-expectation conversion, the
lower-minimax-quantile relation, and the interactive Fano and Le Cam bounds
carry over verbatim. 

\subsection{Private minimax-quantile results}

 The following corollaries are obtained by considering the privatized
ISDM instance with outcome space \(\mathcal Y\), admissible class
\(\mathcal D_\varepsilon^{\mathrm{priv}}\), released-output laws
\(P^{M,\ALG^{\mathrm{priv}}}\), and loss \(L:\mathcal M\times\mathcal Y
\to[0,\infty)\). Each statement is obtained by applying the corresponding result from
Section~\ref{sec:main-results} to the privatized ISDM instance
\[
(\mathcal Y,\mathcal M,\mathcal D_\varepsilon^{\mathrm{priv}},L).
\] No additional argument is
needed beyond this substitution.

\begin{corollary}[Private quantile-to-expectation conversion]
\label{cor:private-quantile-to-expectation}
For every \(\delta\in(0,1]\),
\[
\mathfrak M(\varepsilon)
:=
\inf_{\ALG^{\mathrm{priv}}\in\mathcal D_\varepsilon^{\mathrm{priv}}}
\sup_{M\in\mathcal M}
\mathbb E^{M,\ALG^{\mathrm{priv}}}[L(M,Y)]
\ge
\delta\,\mathfrak M(\delta;\varepsilon).
\]
\end{corollary}

\begin{corollary}[Private lower minimax quantile relation]
\label{cor:private-lower-minimax-quantile-relation}
For every \(\delta\in(0,1]\) and every \(\xi\in(0,\delta)\),
\[
\mathfrak M_-(\delta;\varepsilon)
\le
\mathfrak M(\delta;\varepsilon)
\le
\mathfrak M_-(\delta-\xi;\varepsilon).
\]
Consequently,
\[
\mathfrak M(\delta;\varepsilon)
=
\mathfrak M_-(\delta;\varepsilon)
\]
for all \(\delta\in(0,1]\) except a countable set. Moreover, if for some
\(r\ge 0\),
\[
\inf_{\ALG^{\mathrm{priv}}\in\mathcal D_\varepsilon^{\mathrm{priv}}}
\sup_{M\in\mathcal M}
\mathbb P^{M,\ALG^{\mathrm{priv}}}\bigl(L(M,Y)>r\bigr)
>
\delta,
\]
then
\[
\mathfrak M(\delta;\varepsilon)
\ge
\mathfrak M_-(\delta;\varepsilon)
\ge r.
\]
\end{corollary}

\begin{corollary}[Private high-probability interactive Fano] Fix an
\(f\)-divergence \(D_f\), a prior \(\mu\in\Delta(\mathcal M)\), and a threshold
\(\Delta>0\). For each
\(\ALG^\mathrm{priv}\in\mathcal D^{\mathrm{priv}}_\varepsilon\), let
\(Q_{\ALG^\mathrm{priv}}\in\Delta(\mathcal Y)\) be a reference distribution. Define
\[
\bar\rho_{\Delta,Q_{\ALG^\mathrm{priv}}}
:=
\mathbb P_{M\sim\mu,\;Y\sim Q_{\ALG^\mathrm{priv}}}
\bigl(L(M,Y)\le\Delta\bigr)
\]
and
\[
d_{f,\epsilon}(p)
:=
\begin{cases}
D_f\bigl(\Bern(1-\epsilon)\,\|\,\Bern(p)\bigr),
& p\le 1-\epsilon,\\[0.75ex]
0, & p>1-\epsilon.
\end{cases}
\]
Let
\[
\epsilon^\star_{\mathrm{priv}}
:=
\sup_{\substack{\{Q_{\ALG^\mathrm{priv}}\}_{\ALG^\mathrm{priv}\in\mathcal D^{\mathrm{priv}}_\varepsilon}\\
\epsilon\in[0,1]}}
\left\{
\epsilon:
\sup_{\ALG^\mathrm{priv}\in\mathcal D^{\mathrm{priv}}_\varepsilon}
\left[
\mathbb E_{M\sim\mu}
D_f(P^{M,\ALG^\mathrm{priv}}\|Q_{\ALG^\mathrm{priv}})
-
d_{f,\epsilon}(\bar\rho_{\Delta,Q_{\ALG^\mathrm{priv}}})
\right]
<0
\right\}.
\]
Then, for every \(\delta\in[0,\epsilon^\star_{\mathrm{priv}})\),
\[
\mathfrak M_-(\delta;\varepsilon)\ge \Delta.
\]
\end{corollary}

\begin{corollary}[Private high-probability interactive Le Cam]
\label{cor:private-hp-lecam}
Let \(\delta\in(0,\tfrac12)\), and let \(M_1,M_2\in\mathcal M\). Suppose that
the loss satisfies the uniform separation condition
\[
L(M_1,y)+L(M_2,y)\ge 2\Delta
\qquad
\text{for all }y\in\mathcal Y.
\]
Then the following hold.
\begin{enumerate}[(a)]
    \item If
    \[
    \sup_{\ALG^{\mathrm{priv}}\in\mathcal D_\varepsilon^{\mathrm{priv}}}
    \TV\bigl(
    P^{M_1,\ALG^{\mathrm{priv}}},
    P^{M_2,\ALG^{\mathrm{priv}}}
    \bigr)
    <1-2\delta,
    \]
    then
    \[
    \mathfrak M_-(\delta;\varepsilon)\ge \Delta.
    \]

    \item If
    \[
    \sup_{\ALG^{\mathrm{priv}}\in\mathcal D_\varepsilon^{\mathrm{priv}}}
    \KL\bigl(
    P^{M_1,\ALG^{\mathrm{priv}}}
    \,\|\,
    P^{M_2,\ALG^{\mathrm{priv}}}
    \bigr)
    <
    \log\!\left(\frac{1}{4\delta(1-\delta)}\right),
    \]
    then
    \[
    \mathfrak M_-(\delta;\varepsilon)\ge \Delta.
    \]
\end{enumerate}
\end{corollary}

\subsection{Coordinatewise Gaussian MI privacy}

We now specialize the general MI-private formulation to a common
mechanism class: coordinatewise Gaussian privatization. Let
\[
S=(W_1,\dots,W_d)\in\mathbb R^d
\]
be a \(d\)-dimensional sensitive object. Under coordinatewise Gaussian
privatization, the mechanism releases
\[
\widetilde Y_i=W_i+Z_i,
\qquad
Z_i\stackrel{\mathrm{i.i.d.}}{\sim}\mathcal N(0,\sigma^2),
\qquad
i=1,\dots,d,
\]
where the privacy noises are independent of \(S\) and of any algorithmic
randomization.

A base algorithm \(\ALG\in\mathcal D\) observes only the privatized
coordinates \(\widetilde Y_{1:d}\) and outputs a released object
\(Y\in\mathcal Y\). Depending on the problem, \(Y\) may include the
privatized observations, a final estimate, or an interactive transcript
derived from the privatized observations.

We write
\[
\ALG^{\mathrm{GMI}}\in\mathcal D_\varepsilon^{\mathrm{GMI}}
\]
for the effective procedure induced by coordinatewise Gaussian
privatization and the downstream algorithm. Thus,
\(\ALG^{\mathrm{GMI}}\) denotes the overall mechanism that maps the
sensitive object \(S\) to the released output \(Y\) through the noisy
coordinates \(\widetilde Y_{1:d}\). The class
\(\mathcal D_\varepsilon^{\mathrm{GMI}}\) consists of all such
coordinatewise Gaussian procedures satisfying the uniform MI constraint
\[
\sup_{M\in\mathcal M}
I_{P^{M,\ALG^{\mathrm{GMI}}}}(S;Y)
\le \varepsilon .
\]
Equivalently,
\[
\mathcal D_\varepsilon^{\mathrm{GMI}}
\subseteq
\mathcal D_\varepsilon^{\mathrm{priv}},
\]
where \(\mathcal D_\varepsilon^{\mathrm{priv}}\) denotes the general
MI-private admissible class introduced above.

We denote the corresponding GMI minimax quantiles by
\[
\mathfrak M^{\mathrm{GMI}}(\delta;\varepsilon)
:=
\inf_{\ALG^{\mathrm{GMI}}\in\mathcal D_\varepsilon^{\mathrm{GMI}}}
\sup_{M\in\mathcal M}
\operatorname{Quantile}
\bigl(1-\delta,P^{M,\ALG^{\mathrm{GMI}}},L\bigr),
\]
and
\[
\mathfrak M^{\mathrm{GMI}}_-(\delta;\varepsilon)
:=
\inf\left\{
r\ge0:
\inf_{\ALG^{\mathrm{GMI}}\in\mathcal D_\varepsilon^{\mathrm{GMI}}}
\sup_{M\in\mathcal M}
P^{M,\ALG^{\mathrm{GMI}}}
\bigl(L(M,Y)>r\bigr)
\le \delta
\right\}.
\]
These are the same minimax-quantile definitions as in the 
MI-private formulation, but with the admissible class restricted 
to coordinatewise Gaussian MI-feasible procedures.

\begin{assumption}[Noise floor for coordinatewise Gaussian MI privacy]
\label{ass:noise-floor}
There exists a deterministic function \(\sigma_{\min}^2(\varepsilon,d)\) such
that every feasible private rule \(\ALG^{\mathrm{GMI}} \in D_\varepsilon^{\mathrm{GMI}}\)
satisfies
\[
\sigma^2\ge \sigma_{\min}^2(\varepsilon,d).
\]
\end{assumption}

The second component is a two-point hardness construction whose KL divergence
under the released-output law scales as \[
u^2/(1+\sigma^2).\]

\begin{assumption}[Two-point hardness under Gaussian privatization]
\label{ass:two-point-gaussian}
There exist a constant \(a>0\) and a nondecreasing, left-continuous function
\(\phi:\mathbb R_+\to\mathbb R_+\) such that, for every \(u>0\), there exist
models \(M_1(u),M_2(u)\in\mathcal M\) satisfying, for every released output
\(y\in\mathcal Y\) and every feasible private rule
\(\ALG^{\mathrm{GMI}} \in D_\varepsilon^{\mathrm{GMI}}\),
\[
L(M_1(u),y)+L(M_2(u),y)\ge 2\phi(u),
\]
and
\[
\KL\!\left(
P^{M_1(u),\ALG^{\mathrm{GMI}}}\middle\|P^{M_2(u),\ALG^{\mathrm{GMI}}}
\right)
\le
\frac{a u^2}{1+\sigma^2}.
\]
\end{assumption}

The following theorem gives the principal lower-bound statement for the
MI-private setting under Gaussian privatization.

\begin{theorem}[Coordinatewise Gaussian MI-privatization]
\label{thm:gaussian-private-meta}
Assume that coordinatewise Gaussian privatization is used with \(d\) sensitive
coordinates, and suppose Assumptions~\ref{ass:noise-floor} and
\ref{ass:two-point-gaussian} hold. Let
\[
\Lambda_\delta
:=
\log \Bigl(\frac{1}{4\delta(1-\delta)}\Bigr).
\]
Then, for every \(\delta\in(0,\tfrac12)\),
\[
\mathfrak M^{\mathrm{GMI}}_-(\delta;\varepsilon)
\ge
\phi\!\left(
\sqrt{
\frac{1+\sigma_{\min}^2(\varepsilon,d)}{a}\,
\Lambda_\delta
}
\right).
\]
Consequently,
\[
\mathfrak M^{\mathrm{GMI}}(\delta;\varepsilon)
\ge
\phi\!\left(
\sqrt{
\frac{1+\sigma_{\min}^2(\varepsilon,d)}{a}\,
\Lambda_\delta
}
\right).
\]
\end{theorem}
\begin{proof}[Proof sketch]
    The conditions of Corollary~\ref{cor:private-hp-lecam}(b) are verified using
Assumptions~1 and~2. For the full proof, we refer the reader to the appendix.
\end{proof}

\begin{remark}
    Theorem~\ref{thm:gaussian-private-meta} reduces private minimax quantile lower
bounds to two tasks: derive the privacy-induced noise floor
\(\sigma_{\min}^2(\varepsilon,d)\), and construct a two-point family satisfying
the separation and KL conditions.
\end{remark}

In the next section, we verify these
conditions for Gaussian mean estimation and two-armed Gaussian bandits.

\section{Applications}
\label{sec:applications}

We now illustrate the theory in five examples. We begin with two classical
problems, Gaussian mean estimation and two-armed Gaussian bandits, to show how
the Le Cam method recovers standard \(\delta\)-explicit quantile scalings. We
then give a \(K\)-armed Gaussian bandit example showing how the interactive
Fano method handles multiple possible best arms.
Finally, we turn to the corresponding coordinatewise Gaussian MI-private
examples for mean estimation and two-armed bandits.
\subsection{Classical Gaussian mean estimation}
\label{sec:nonprivate-mean}

We first consider one-dimensional Gaussian mean estimation in the
classical setting. This example is non-interactive, but it fits naturally
within ISDM and serves as an illustration of
Theorem~\ref{thm:hp-lecam}.

Let
\[
\mathcal M=\{\theta:\theta\in\mathbb R\},
\qquad
X_1,\dots,X_n \stackrel{\mathrm{i.i.d.}}{\sim} \mathcal N(\theta,1).
\]
An estimator \(\ALG\) maps the sample \(X_{1:n}\) to an estimate
\(\hat\theta=\ALG(X_{1:n})\). The final outcome is
\[
X=(X_{1:n},\hat\theta).
\]
We take the loss to be the squared estimation error

\[
L(\theta,X)=(\hat\theta-\theta)^2.
\]

The following theorem provides a lower bound for the minimax quantile for the Gaussian estimation problem under the squared-error loss. 

\begin{theorem}[Gaussian mean estimation under squared-error]
\label{thm:nonprivate-mean}
For every \(\delta\in(0,\tfrac12)\),
\[
\mathfrak M_-(\delta)
\ge
\frac{1}{2n}\,
\log\!\Bigl(\frac{1}{4\delta(1-\delta)}\Bigr).
\]
Consequently, for all such \(\delta\),
\[
\mathfrak M(\delta)
\ge
\frac{1}{2n}\,
\log\!\Bigl(\frac{1}{4\delta(1-\delta)}\Bigr).
\]
\end{theorem}
\begin{proof}[Proof sketch]
    For the full proof, we refer the reader to the appendix.
\end{proof}

\begin{remark}[Relation to known Gaussian mean-estimation quantile lower bound]
Theorem~\ref{thm:nonprivate-mean} recovers the one-dimensional Gaussian
mean-estimation quantile rate under squared-error. Ma, Verchand and
Samworth~\cite[Prop.~10]{ma2024high} consider observations
\[
X_i\sim N_d(\theta,\Sigma), \qquad i=1,\dots,n,
\]
where \(\theta\in\mathbb R^d\) is the unknown mean vector and
\(\Sigma\in\mathbb R^{d\times d}\) is the covariance matrix. Under squared
Euclidean loss, their minimax quantile is, up to universal constants, of order
\[
\frac{\operatorname{tr}(\Sigma)}{n}
+
\frac{\|\Sigma\|_{\mathrm{op}}\log(1/\delta)}{n},
\]
where \(\operatorname{tr}(\Sigma)\) is the trace and
\(\|\Sigma\|_{\mathrm{op}}\) is the largest eigenvalue of \(\Sigma\). In the
scalar unit-variance case considered here, \(d=1\) and \(\Sigma=1\), so this
becomes
\[
\frac{1}{n}+\frac{\log(1/\delta)}{n}.
\]
For small \(\delta\), i.e., confidence level \(1-\delta\) close to 1, the
second term gives the dominant tail dependence, matching the
\(\log(1/\delta)/n\) scaling in Theorem~\ref{thm:nonprivate-mean}. Thus, our
result recovers the corresponding lower-bound rate through the ISDM formulation
and the interactive Le Cam's method. 
\end{remark}

\subsection{Classical two-armed Gaussian bandits}
\label{sec:warmup-bandit}

Next, we study the classical two-armed Gaussian bandit, which provides an
interactive illustration of the theory. In order for the minimax quantile to be
finite, we work with a bounded-mean Gaussian bandit class.

Consider the model class
\[
\mathcal M
=
\Bigl\{
(\mathcal N(\mu_1,1),\mathcal N(\mu_2,1)):
(\mu_1,\mu_2)\in[0,1]^2
\Bigr\}.
\]
At each round \(t=1,\dots,T\), the learner chooses an arm
\(A_t\in\{1,2\}\), then observes a reward
\[
R_t\,|\,A_t=a \sim \mathcal N(\mu_a,1).
\]
Let
\[
H_t=(A_1,R_1,\dots,A_t,R_t)
\]
denote the interaction history up to time \(t\), and let \(\ALG\) be any
possibly randomized adaptive policy. The final outcome is the full transcript
\[
X=H_T.
\]
We take the loss to be the pseudo-regret
\[
L(M,X)
=
\sum_{t=1}^T \bigl(\mu^\star(M)-\mu_{A_t}\bigr),
\]
where
\[
\mu^\star(M)=\max\{\mu_1,\mu_2\}.
\]
The following theorem provides a lower bound for the minimax quantile for the Gaussian bandit problem.  
\begin{theorem}[Two-armed Gaussian bandit with bounded means]
\label{thm:nonprivate-two-arm}
For every \(\delta\in(0,\tfrac12)\), let
\[
\Lambda_\delta
:=
\log\!\Bigl(\frac{1}{4\delta(1-\delta)}\Bigr).
\]
Then
\[
\mathfrak M_-(\delta)
\ge
\frac{T}{2}
\min\!\left\{
1,\,
\sqrt{\frac{2\Lambda_\delta}{T}}
\right\}.
\]
Consequently, for all such \(\delta\),
\[
\mathfrak M(\delta)
\ge
\frac{T}{2}
\min\!\left\{
1,\,
\sqrt{\frac{2\Lambda_\delta}{T}}
\right\}.
\]
In particular, whenever \(\Lambda_\delta\le T/2\),
\[
\mathfrak M(\delta)
\ge
\sqrt{
\frac{T}{2}\,
\log\!\Bigl(\frac{1}{4\delta(1-\delta)}\Bigr)
}.
\]

\end{theorem}
\begin{proof}[Proof sketch]
    For the full proof, we refer the reader to the appendix.
\end{proof}

\begin{remark}
    Theorem~\ref{thm:nonprivate-two-arm} shows that the interactive Le Cam's method
yields a \(\delta\)-explicit strict minimax-quantile lower bound for bounded-mean
two-armed Gaussian bandits. The hard-instance construction parallels the bounded
Gaussian constructions used by Lattimore and Szepesvári~\cite[Chs.~13 and
17]{lattimore2020bandit}, although their results are stated as high-probability
tail lower bounds rather than strict minimax-quantile bounds. In the regime
\[
\log\!\Bigl(\frac{1}{4\delta(1-\delta)}\Bigr)\le T/2,
\]
the lower bound has order \(\sqrt{T\log(1/\delta)}\). For smaller \(\delta\)
levels, the bound is truncated at order \(T\), which is the natural scale in the
bounded-gap model: since \(\mu_1,\mu_2\in[0,1]\), the per-round pseudo-regret is
at most one.
\end{remark}

\subsection{\(K\)-armed Gaussian bandits using interactive Fano}
\label{sec:karmed-fano-bandit}

We now give an application of the high-probability
interactive Fano method. The preceding two-armed lower bound is based on a
two-point Le Cam construction. Such a construction captures the dependence on
the horizon \(T\) and the confidence level \(\delta\), but it does not expose
the exploration cost associated with \(K\) possible best arms. The following
result uses a \(K\)-point family, one model for each possible best arm, and
therefore recovers the usual \(\sqrt{K}\) factor in the minimax bandit rate.

Consider the \(K\)-armed Gaussian bandit model
\[
\mathcal M_K
=
\left\{
\bigl(\mathcal N(\mu_1,1),\dots,\mathcal N(\mu_K,1)\bigr):
\mu=(\mu_1,\dots,\mu_K)\in[0,1]^K
\right\}.
\]
At each round \(t=1,\dots,T\), the learner chooses an arm
\[
A_t\in[K],
\]
and then observes
\[
R_t\,|\,A_t=a \sim \mathcal N(\mu_a,1).
\]
Let
\[
H_T=(A_1,R_1,\dots,A_T,R_T)
\]
denote the full interaction transcript. The final outcome is
\[
X=H_T.
\]
For \(M\in\mathcal M_K\), write
\[
\mu^\star(M):=\max_{a\in[K]}\mu_a.
\]
The loss is the pseudo-regret
\[
L(M,H_T)
=
\sum_{t=1}^T
\bigl(\mu^\star(M)-\mu_{A_t}\bigr).
\]

\begin{theorem}[Fano lower bound for \(K\)-armed Gaussian bandits]
\label{thm:karmed-fano-bandit}
Let \(K\ge 2\), \(T\ge 2\), and \(\delta\in(0,1-1/K)\). Define
\[
\Lambda_{K,\delta}
:=
\KL\left(\frac1K\,\middle\|\,1-\delta\right),
\]
where
\[
\KL(p\|q)
:=
p\log\frac{p}{q}
+
(1-p)\log\frac{1-p}{1-q}
\]
is the binary relative entropy. Then the lower minimax quantile for the
\(K\)-armed Gaussian bandit pseudo-regret satisfies
\[
\mathfrak M_-(\delta)
\ge
\frac{T-1}{2}
\min\left\{
1,
\sqrt{
\frac{2K}{T}
\,
\KL\left(\frac1K\,\middle\|\,1-\delta\right)
}
\right\}.
\]
\end{theorem}
\begin{proof}[Proof sketch]
The proof applies Theorem~\ref{thm:hp-interactive-fano} to the \(K\) alternatives
in which exactly one arm has mean \(\alpha\), using the all-zero bandit
transcript law \(Q_\ALG=P_0^\ALG\) as the reference law and reverse KL as the
\(f\)-divergence. The reference success probability is at most \(1/K\), while
the adaptive KL decomposition gives
\[
\frac1K\sum_{j=1}^K \KL(P_0^\ALG\|P_j^\ALG)
=
\frac{\alpha^2T}{2K}.
\]
Optimizing over \(\alpha\) and letting the auxiliary Fano level
\(\epsilon\downarrow\delta\) gives the claim. The full proof is given in the appendix.
\end{proof}

\begin{remark}[Scaling of the \(K\)-armed Gaussian bandit lower bound]
\label{rem:fano-kl-confidence-term}
When the minimum in Theorem~\ref{thm:karmed-fano-bandit} is attained by its
second argument, the lower bound is
\[
\frac{T-1}{2}
\sqrt{
\frac{2K}{T}
\KL\left(\frac1K\,\middle\|\,1-\delta\right)
},
\]
which has order
\[
\sqrt{
KT\,
\KL\left(\frac1K\,\middle\|\,1-\delta\right)
}.
\]
The binary relative entropy term satisfies
\[
K\KL\left(\frac1K\,\middle\|\,1-\delta\right)
=
(K-1)\log\left(\frac1\delta\right)
-\log K
-\log(1-\delta)
+
(K-1)\log\left(1-\frac1K\right).
\]
Hence, whenever
\[
(K-1)\log(1/\delta)\gg \log K,
\]
this lower bound has leading order
\[
\sqrt{(K-1)T\log(1/\delta)}.
\]

For \(K=2\),
\[
\KL\left(\frac12\,\middle\|\,1-\delta\right)
=
\frac12
\log\left(\frac{1}{4\delta(1-\delta)}\right).
\]
Thus the bound becomes
\[
\sqrt{
T\log\left(\frac{1}{4\delta(1-\delta)}\right)
}
\]
up to universal constants, matching the confidence dependence of the
two-point Le Cam bound in Theorem \ref{thm:nonprivate-two-arm}.
\end{remark}

\subsection{Gaussian mean estimation under coordinatewise Gaussian MI-privatization}
\label{sec:private-mean}

We now return to Gaussian mean estimation, but with MI privacy imposed on the
sample vector. Let
\[
\mathcal M=\{\theta:\theta\in\mathbb R\},
\qquad
X_1,\dots,X_n \stackrel{\mathrm{i.i.d.}}{\sim} \mathcal N(\theta,1).
\]
The sensitive object is
\[
S=(X_1,\dots,X_n)\in\mathbb R^n.
\]

Fix a privacy budget \(\varepsilon>0\). We denote by
\(\ALG^{\mathrm{GMI}}\) a private estimator whose released output
is \(Y\). The estimator is admissible if its released output satisfies the MI constraint
\[
\sup_{\theta\in\mathcal M}
I_{P^{\theta,\ALG^{\mathrm{GMI}}}}(S;Y)
\le \varepsilon.
\]

Under coordinatewise Gaussian privatization, the raw sample is first passed
through the privacy mechanism
\[
\widetilde Y_i=X_i+Z_i,
\qquad
Z_i \stackrel{\mathrm{i.i.d.}}{\sim}\mathcal N(0,\sigma^2),
\qquad
i=1,\dots,n,
\]
where the privacy noises \(Z_1,\dots,Z_n\) are independent of the sample. The
estimator is constrained to observe only the privatized sample
\(\widetilde Y_{1:n}\), and outputs
\[
\hat\theta=\ALG(\widetilde Y_{1:n}).
\]
Thus, the raw sample \(X_{1:n}\) is not directly available to the estimator.
The released output is
\[
Y=(\widetilde Y_{1:n},\hat\theta),
\]
and the loss is the squared-error
\[
L(\theta,Y)=(\hat\theta-\theta)^2.
\]

In this Gaussian mechanism, the MI constraint imposes the noise floor
\[
\sigma^2\ge \sigma_{\min}^2(\varepsilon,n),
\qquad
\sigma_{\min}^2(\varepsilon,n)
=
\frac{1}{e^{2\varepsilon/n}-1}.
\]
For each \(\theta\in\mathcal M\), we write
\[
Y\sim P_Y^{\theta,\ALG^{\mathrm{GMI}}},
\]
where \(P_Y^{\theta,\ALG^{\mathrm{GMI}}}\) is the law induced by the raw
sample, the Gaussian privacy mechanism, and the estimator acting only on the
privatized sample.

The following theorem provides a minimax-quantile lower bound for Gaussian mean
estimation under squared-error loss with coordinatewise Gaussian
MI privacy.

\begin{theorem}[Gaussian mean estimation under coordinatewise Gaussian MI privacy]
\label{thm:private-mean}
For every \(\delta\in(0,\tfrac12)\),
\[
\mathfrak M^{\mathrm{GMI}}_-(\delta;\varepsilon)
\ge
\frac{1+\sigma_{\min}^2(\varepsilon,n)}{2n}\,
\log\!\Bigl(\frac{1}{4\delta(1-\delta)}\Bigr),
\]
where
\[
\sigma_{\min}^2(\varepsilon,n)
=
\frac{1}{e^{2\varepsilon/n}-1}.
\]
Consequently,
\[
\mathfrak M^{\mathrm{GMI}}(\delta;\varepsilon)
\ge
\frac{1+\sigma_{\min}^2(\varepsilon,n)}{2n}\,
\log\!\Bigl(\frac{1}{4\delta(1-\delta)}\Bigr).
\]
\end{theorem}

\begin{proof}[Proof sketch]
We verify the hypotheses of Theorem~\ref{thm:gaussian-private-meta} for the
Gaussian mean-estimation model under squared-error loss. The MI constraint
yields the noise floor
\(\sigma^2\ge\sigma_{\min}^2(\varepsilon,n)\), while the standard two-point
Gaussian construction verifies the required separation and KL conditions. The
full proof is given in the appendix.
\end{proof}

\begin{remark}
Theorem~\ref{thm:private-mean} is the squared-error counterpart of
Theorem~\ref{thm:nonprivate-mean} under coordinatewise Gaussian
MI privacy. Relative to the non-private setting, the lower bound is
inflated by the factor \(1+\sigma_{\min}^2(\varepsilon,n)\), where
\[
\sigma_{\min}^2(\varepsilon,n)
=
\frac{1}{e^{2\varepsilon/n}-1}.
\]
Since \(\sigma_{\min}^2(\varepsilon,n)\) is decreasing in the privacy budget
\(\varepsilon\), stronger privacy requirements (smaller \(\varepsilon\))
necessitate larger privatization noise and therefore yield larger minimax
quantiles. Thus, the degradation in statistical performance is governed
precisely by the variance inflation induced by the Gaussian privacy mechanism.
\end{remark}

\subsection{Two-armed Gaussian bandits under coordinatewise Gaussian MI privacy}
\label{sec:private-bandit}

We return to the two-armed Gaussian bandit, but now impose MI privacy on the
reward sequence observed during the interaction. As in the classical case, we
work with bounded means. Let
\[
\mathcal M
=
\{(\mu_1,\mu_2):\mu_1,\mu_2\in[0,1]\},
\qquad
R_t\,|\,A_t=a \sim \mathcal N(\mu_a,1).
\]
The sensitive object is the realized raw reward sequence
\[
S=(R_1,\dots,R_T)\in\mathbb R^T.
\]

Fix a privacy budget \(\varepsilon>0\). We denote by
\(\ALG^{\mathrm{GMI}}\) an effective private bandit algorithm whose released
transcript is \(Y\). The algorithm is admissible if its released transcript
satisfies the uniform MI constraint
\[
\sup_{M\in\mathcal M}
I_{P^{M,\ALG^{\mathrm{GMI}}}}(S;Y)
\le \varepsilon.
\]

Under coordinatewise Gaussian privatization, the raw reward \(R_t\) is not
directly observed by the learning policy. Instead, after the learner chooses an
action \(A_t\), the environment generates the raw reward \(R_t\), and the
privacy mechanism releases only
\[
\widetilde Y_t = R_t + Z_t,
\qquad
Z_t \stackrel{\mathrm{i.i.d.}}{\sim} \mathcal N(0,\sigma^2),
\qquad
t=1,\dots,T,
\]
where the privacy noises are independent of the rewards and of the algorithm's
internal randomization. 

The policy component of \(\ALG^{\mathrm{GMI}}\) is a sequence of stochastic
kernels
\[
\pi_t:\widetilde{\mathcal H}_{t-1}\to\Delta(\{1,2\}),
\qquad t=1,\dots,T,
\]
where \(\widetilde{\mathcal H}_{t-1}\) is the space of privatized histories up
to time \(t-1\). Given the realized privatized history
\[
\widetilde H_{t-1}
=
(A_1,\widetilde Y_1,\dots,A_{t-1},\widetilde Y_{t-1}),
\]
the action is sampled according to
\[
A_t \sim \pi_t(\cdot \mid \widetilde H_{t-1}).
\]
Thus, the policy never observes the raw rewards \(R_1,\dots,R_T\) and it observes
only their privatized versions.

The released transcript is
\[
Y=(A_1,\widetilde Y_1,\dots,A_T,\widetilde Y_T).
\]

In this Gaussian mechanism, the MI constraint imposes the noise floor
\[
\sigma^2\ge \sigma_{\min}^2(\varepsilon,T),
\qquad
\sigma_{\min}^2(\varepsilon,T)
=
\frac{1}{e^{2\varepsilon/T}-1}.
\]
For each \(M\in\mathcal M\), we write
\[
Y\sim P_Y^{M,\ALG^{\mathrm{GMI}}},
\]
where \(P_Y^{M,\ALG^{\mathrm{GMI}}}\) is the released-transcript law induced
by the model, the Gaussian privacy mechanism, and the policy acting only on the
privatized history.

The loss is the pseudo-regret
\[
L(M,Y)
=
\sum_{t=1}^T \bigl(\mu^\star(M)-\mu_{A_t}\bigr),
\qquad
\mu^\star(M)=\max\{\mu_1,\mu_2\}.
\]
Although the loss is written as a function of \(Y\), it depends on the released
transcript only through the action sequence \(A_1,\dots,A_T\). The raw rewards
are used only to generate the privatized observations and to define the
sensitive object protected by the MI constraint.

The following theorem provides a minimax-quantile lower bound for the
two-armed Gaussian bandit problem under coordinatewise Gaussian
MI privacy.
\begin{theorem}[Two-armed Gaussian bandits under coordinatewise Gaussian MI privacy]
\label{thm:private-bandit}
For every \(\delta\in(0,\tfrac12)\), let
\[
\Lambda_\delta
:=
\log\!\Bigl(\frac{1}{4\delta(1-\delta)}\Bigr),
\]
and
\[
\sigma_{\min}^2(\varepsilon,T)
=
\frac{1}{e^{2\varepsilon/T}-1}.
\]
Then
\[
\mathfrak M^{\mathrm{GMI}}_-(\delta;\varepsilon)
\ge
\frac{T}{2}
\min\!\left\{
1,\,
\sqrt{
\frac{2(1+\sigma_{\min}^2(\varepsilon,T))\Lambda_\delta}{T}
}
\right\}.
\]
Consequently,
\[
\mathfrak M^{\mathrm{GMI}}(\delta;\varepsilon)
\ge
\frac{T}{2}
\min\!\left\{
1,\,
\sqrt{
\frac{2(1+\sigma_{\min}^2(\varepsilon,T))\Lambda_\delta}{T}
}
\right\}.
\]
In particular, whenever
\[
(1+\sigma_{\min}^2(\varepsilon,T))\Lambda_\delta\le \frac{T}{2},
\]
we have
\[
\mathfrak M^{\mathrm{GMI}}_-(\delta;\varepsilon)
\ge
\sqrt{
\frac{(1+\sigma_{\min}^2(\varepsilon,T))T}{2}
\log\!\Bigl(\frac{1}{4\delta(1-\delta)}\Bigr)
},
\]
and the same lower bound holds for \(\mathfrak M^{\mathrm{GMI}}(\delta;\varepsilon)\).
\end{theorem}
\begin{proof}[Proof sketch]
We verify the conditions of Corollary~\ref{cor:private-hp-lecam}(b) for the
MI-private Gaussian bandit problem. The MI constraint yields the noise floor
\(\sigma^2\ge \sigma_{\min}^2(\varepsilon,T)\), while the symmetric two-point
bandit construction verifies the required separation and KL
indistinguishability conditions. The full proof is given in the appendix.
\end{proof}

\begin{remark}
   Theorem~\ref{thm:private-bandit} is the coordinatewise Gaussian MI-privatized counterpart of
Theorem~\ref{thm:nonprivate-two-arm}. Privacy enters through the Gaussian
variance-inflation factor \(1+\sigma_{\min}^2(\varepsilon,T)\), now propagated
through the horizon \(T\). In the regime
\[
(1+\sigma_{\min}^2(\varepsilon,T))
\log\!\Bigl(\frac{1}{4\delta(1-\delta)}\Bigr)
\le \frac{T}{2},
\]
the lower bound has order
\(
\sqrt{
(1+\sigma_{\min}^2(\varepsilon,T))T\log(1/\delta)
}.
\)
For smaller \(\delta\) or stronger privacy constraints (larger \(\sigma_{\min}\)), the lower bound
is truncated at order \(T\), which is the natural scale in the bounded-mean
model: since \(\mu_1,\mu_2\in[0,1]\), the per-round pseudo-regret is at most one. 
\end{remark}

\subsection{Comparison of the bounds}
\label{sec:comparison}

The examples above illustrate two complementary uses of the quantile framework:
the Le Cam argument gives two-point lower bounds for estimation, two-armed
bandits, and their coordinatewise Gaussian MI-private counterparts, while the Fano argument captures
the multi-arm exploration in the \(K\)-armed bandit example. In the
classical Gaussian mean-estimation example with squared-error, the lower bound
has scale
\[
\frac{1}{n}
\log\!\Bigl(\frac{1}{4\delta(1-\delta)}\Bigr).
\]
In contrast, the two-armed bandit example has scale
\[
\min\!\left\{
T,\,
\sqrt{
T\log\!\Bigl(\frac{1}{4\delta(1-\delta)}\Bigr)
}
\right\}.
\]
Thus the estimation and bandit examples exhibit a distinction between the
non-interactive setting, where the squared-error quantile scales as
\(\log(1/\delta)/n\), and the sequential setting, where the high-confidence regret
scale is \(\sqrt{T\log(1/\delta)}\) until it reaches the bounded-gap ceiling of
order \(T\).

The private examples have the same structure, but with an additional
variance-inflation factor induced by the Gaussian privacy mechanism. For
mean estimation, Theorem~\ref{thm:private-mean} replaces the non-private
factor \(1/n\) by
\[
\frac{1+\sigma_{\min}^2(\varepsilon,n)}{n},
\qquad
\sigma_{\min}^2(\varepsilon,n)
=
\frac{1}{e^{2\varepsilon/n}-1}.
\]
Thus, under squared-error, the privacy mechanism inflates the lower bound multiplicatively
by
\[
1+\sigma_{\min}^2(\varepsilon,n).
\]
For the two-armed bandit, Theorem~\ref{thm:private-bandit} similarly replaces
the non-private variance scale by
\[
1+\sigma_{\min}^2(\varepsilon,T),
\qquad
\sigma_{\min}^2(\varepsilon,T)
=
\frac{1}{e^{2\varepsilon/T}-1},
\]
leading to the lower-bound scale
\[
\min\!\left\{
T,\,
\sqrt{
(1+\sigma_{\min}^2(\varepsilon,T))T
\log\!\Bigl(\frac{1}{4\delta(1-\delta)}\Bigr)
}
\right\}.
\]
The truncation by \(T\) in the bandit bounds reflects the bounded-mean
formulation: since \(\mu_1,\mu_2\in[0,1]\), the per-round pseudo-regret is at
most one.

As \(\varepsilon\to\infty\), one has
\[
\sigma_{\min}^2(\varepsilon,d)\to 0,
\]
so the private lower bounds reduce to their non-private counterparts. As
\(\varepsilon\) decreases, the required privatization noise increases, and the
corresponding minimax quantile lower bounds worsen. These examples therefore highlight the role of the Gaussian-privatization
template and the underlying Le Cam argument: once the
privacy-induced noise floor and the two-point KL bound are identified, the
strict minimax-quantile lower bound follows in a unified way across both
estimation and interactive bandit problems.

\subsection{Upper bounds and related achievability results}
\label{sec:upper-bounds-comparison}

We now compare the preceding lower bounds with known upper bounds. The
comparisons in this section should be interpreted with some care. Only the
Gaussian mean-estimation comparison is a direct match of model, loss, and
quantile criterion. The bandit and privacy results discussed below are related
achievability results: they establish the high-confidence rates in
nearby problems, but they differ from our theorems either in the reward class, the
regret notion, or the privacy definition. In particular, any upper bound stated in the weak-tail form
\[
\sup_{M\in\mathcal M}\mathbb P_{M,\mathrm{ALG}}\{L(M,X)\ge r_\delta\}\le \delta
\]
also implies the strict-tail bound
\[
\sup_{M\in\mathcal M}\mathbb P_{M,\mathrm{ALG}}\{L(M,X)> r_\delta\}\le \delta,
\]
since \(\{L(M,X)>r_\delta\}\subseteq\{L(M,X)\ge r_\delta\}\), and hence gives the strict minimax-quantile upper bound \(M(\delta)\le r_\delta\).

For classical Gaussian mean estimation under squared-error, the lower bound in
Theorem~\ref{thm:nonprivate-mean} is tight up to universal constants. Ma,
Verchand and Samworth~\cite[Prop.~10]{ma2024high} show that, for
\(X_i\sim\mathcal N_d(\theta,\Sigma)\), the minimax quantile under squared
Euclidean loss has order
\[
\frac{\operatorname{tr}(\Sigma)}{n}
+
\frac{\|\Sigma\|_{\mathrm{op}}\log(1/\delta)}{n}.
\]
In the one-dimensional case \(\Sigma=1\), this becomes
\[
\mathfrak M(\delta)\asymp \frac{\log(1/\delta)}{n},
\]
which matches Theorem~\ref{thm:nonprivate-mean}. The matching upper bound is
achieved by the sample mean. This is the cleanest comparison in the present
paper, since the model, loss, and quantile criterion coincide.

For \(K\)-armed bandits, the situation is more nuanced. Our lower bound is stated
for bounded-mean Gaussian stochastic bandits and pseudo-regret. The cleanest
matching high-probability upper bounds with
\(\sqrt{KT\log(1/\delta)}\) dependence are available in nearby bounded-reward
stochastic formulations. In particular, fixed-confidence EXP3-type algorithms
satisfy high-probability regret bounds of order
\[
O\!\left((1+\log K)\sqrt{KT\log(K/\delta)}\right),
\]
which reduces to \(O(\sqrt{T\log(1/\delta)})\) for \(K=2\)
\cite{seldin2013evaluation}. Thus, for bounded-reward two-armed stochastic
bandits, the \(\sqrt{T\log(1/\delta)}\) rate is achievable up to constants.

For bounded-mean Gaussian rewards, however, the rewards themselves are
unbounded, so bounded-reward upper bounds do not apply directly. A directly
relevant stochastic distributional-regret result is due to Lee and
Oh~\cite{lee2026unified}, who obtain, for sub-Gaussian stochastic MAB with
bounded value scale, the uniform-in-\(\delta\) bound
\[
O\!\left(\sqrt{KT}\log(1/\delta)\right)
\]
together with \(O(\sqrt{KT})\) expected regret. This is a 
distributional-regret upper bound, but its confidence dependence is
\(\log(1/\delta)\), rather than \(\sqrt{\log(1/\delta)}\). Thus the
bounded-reward stochastic theory is rate-consistent with our two-arm lower
bound, while for bounded-mean Gaussian rewards we are not aware of a classical
UCB-type algorithm/theorem giving the sharp
\(O(\sqrt{KT\log(1/\delta)})\) strict quantile upper bound.

Aliakbarpour et al.~\cite{aliakbarpour2025high} also study high-probability
minimax quantiles under privacy constraints, but in the heterogeneous local
differential privacy (LDP) model rather than under MI privacy with Gaussian
privatization. Their mean-estimation setting is bounded rather than Gaussian, and the comparison is therefore qualitative. For one-dimensional
mean estimation with user-specific privacy levels
\(\varepsilon_1,\dots,\varepsilon_n\), they show that the optimal
high-probability squared-error scale is, up to universal constants,
\[
\min\!\left\{
\frac{\log(1/\delta)}{\sum_{i=1}^n \varepsilon_i^2},
1
\right\},
\]
under the high-privacy condition \(\varepsilon_i\le 1\). In particular, their
upper bounds are achieved by heterogeneous LDP mechanisms, and their matching
lower bound shows that this rate is minimax-optimal up to constants.

In our MI-private Gaussian mean-estimation bound, privacy enters through the
variance-inflation factor
\[
1+\sigma_{\min}^2(\varepsilon,n),
\qquad
\sigma_{\min}^2(\varepsilon,n)
=
\frac{1}{e^{2\varepsilon/n}-1}.
\]
Thus, stronger MI privacy, corresponding to smaller \(\varepsilon\), increases
the required Gaussian noise and worsens the minimax-quantile lower bound. In
heterogeneous LDP, stronger privacy corresponds to smaller local parameters
\(\varepsilon_i\), which decreases the aggregate privacy strength
\(\sum_i \varepsilon_i^2\) and worsens the optimal high-probability rate. Thus,
both theories exhibit the same qualitative phenomenon: stronger privacy leads
to worse high-confidence estimation performance, although the privacy notions
and parameters are not directly comparable.

For private bandits, the comparison is even more indirect. Existing work on
differentially private and locally differentially private bandits studies regret
under DP, LDP, heavy-tailed private rewards, and
concentrated or interactive privacy constraints
\cite{tossou2016algorithms,basu2019differential,tao2022optimal,azize2024concentrated}.
These results typically provide expected regret bounds or 
high-probability regret bounds under DP/LDP-type privacy constraints. They do not
give, in the same model, a matching strict minimax-quantile upper bound for the
MI-private Gaussian mechanism studied here. Consequently,
Theorem~\ref{thm:private-bandit} should be viewed as a lower bound for the
MI-private quantile formulation rather than as a result currently matching an existing private bandit algorithmic upper bound under the same privacy notion.

%

\section{Conclusion and Future Work}
\label{sec:conclusion}

We developed a \(\delta\)-explicit minimax-quantile theory for interactive statistical decision making. The main contributions were a quantile-to-expectation conversion, a structural relation between strict and lower minimax quantiles, and high-probability interactive versions of Fano's and Le Cam's methods. Together, these results provide a converse toolkit for tail-sensitive lower bounds under interaction.

We then showed that mutual-information-based private decision problems can be modeled as ISDM instances with a privacy-constrained admissible class. Once this constrained-class formulation is in place, the ISDM quantile machinery transfers directly.

For coordinatewise Gaussian MI-privatization, we derived a two-point template that isolates the role of privacy through a noise-floor condition and a
KL-based indistinguishability condition. Instantiating this template yielded minimax-quantile lower bounds for Gaussian mean estimation and two-armed Gaussian bandits under coordinatewise Gaussian MI privacy. We also derived a
classical \(K\)-armed Gaussian bandit application showing that the interactive Fano theorem captures multi-arm scaling beyond two-point Le Cam
arguments.

Several directions remain open. First, our treatment of privacy focused on mutual-information constraints and Gaussian privatization; extending the same minimax-quantile program to other privacy notions and other mechanism classes is a natural next step. Next, on the interactive side, it would be valuable to connect the present strict-quantile approach more directly to DEC-based lower bounds and to derive \(\delta\)-explicit quantile results for broader classes of bandit and reinforcement learning problems. Another direction is to develop matching algorithmic upper bounds for the MI-private problems considered here. Finally, it would be interesting to study whether local minimax lower-bound techniques of the type developed in non-interactive settings can be transported to interactive and privacy-constrained minimax quantiles in a similarly systematic way \cite[Theorem~8]{ma2024high}.
\IEEEpeerreviewmaketitle
 
\bibliographystyle{IEEEtran}

\newpage
\bibliography{references}

\newpage
\appendix
\section{Proofs}
\begin{proof}[\textbf{Proof of Theorem~\ref{thm:quantile-to-expectation}}]
Fix \(M\in\mathcal M\) and \(\ALG\in\mathcal D\), and let
\[
r_{M,\ALG}
:=
\Quantile(1-\delta,\mathbb P^{M,\ALG},L).
\]
By the definition of the strict quantile, for every \(s<r_{M,\ALG}\)
we have \(\mathbb P(L>s)>\delta\). Letting \(s\uparrow r_{M,\ALG}\) implies
\(\mathbb P(L\ge r_{M,\ALG})\ge\delta\).
Since \(L(M,X)\ge 0\), we have
\[
\mathbb E^{M,\ALG}[L(M,X)]
\ge
r_{M,\ALG}\,
\mathbb P^{M,\ALG}\bigl(L(M,X)\ge r_{M,\ALG}\bigr)
\ge
\delta\, r_{M,\ALG}.
\]
Therefore,
\[
\mathbb E^{M,\ALG}[L(M,X)]
\ge
\delta\,
\Quantile(1-\delta,\mathbb P^{M,\ALG},L).
\]
Taking the supremum over \(\mathcal M\) and then the infimum over \(\mathcal D\) yields the claim.
\end{proof}
\begin{proof}[\textbf{Proof of Theorem~\ref{thm:lower-minimax-quantile-relation}}]
    The proof adapts the argument in the non-interactive setting of \cite[Thm.~4]{ma2024high} and uses the induced laws \(\mathbb P^{M,\ALG}\).

For the first inequality, if \(r>\mathfrak M(\delta)\), then there exists an algorithm \(\ALG\) such that for every \(M\in\mathcal M\),
\[
\Quantile(1-\delta,\mathbb P^{M,\ALG},L)\le r.
\]
Hence,
\[
\sup_{M\in\mathcal M}\mathbb P^{M,\ALG}(L(M,X)>r)\le \delta,
\]
which implies \(\mathfrak M_{-}(\delta)\le r\). Letting \(r\downarrow \mathfrak M(\delta)\) gives
\[
\mathfrak M_{-}(\delta)\le \mathfrak M(\delta).
\]

For the second inequality, fix \(\xi\in(0,\delta)\) and let \(r>\mathfrak M_{-}(\delta-\xi)\). Then there exists \(\ALG\) such that
\[
\sup_{M\in\mathcal M}\mathbb P^{M,\ALG}(L(M,X)>r)\le \delta-\xi < \delta.
\]
Therefore,
\[
\sup_{M\in\mathcal M}\Quantile(1-\delta,\mathbb P^{M,\ALG},L)\le r,
\]
so \(\mathfrak M(\delta)\le r\). Letting \(r\downarrow \mathfrak M_{-}(\delta-\xi)\) gives
\[
\mathfrak M(\delta)\le \mathfrak M_{-}(\delta-\xi).
\]

Since \(\delta\mapsto \mathfrak M_{-}(\delta)\) is non-increasing, it has at most countably many discontinuities, and the equality \(\mathfrak M(\delta)=\mathfrak M_{-}(\delta)\) follows for all other \(\delta\).

Finally, if
\[
\inf_{\ALG\in\mathcal D}\sup_{M\in\mathcal M}\mathbb P^{M,\ALG}(L(M,X)>r)>\delta,
\]
then by the definition of \(\mathfrak M_{-}(\delta)\) we must have \(\mathfrak M_{-}(\delta)\ge r\). The first inequality then gives
\[
\mathfrak M(\delta)\ge \mathfrak M_{-}(\delta)\ge r.
\]
\end{proof}

\begin{proof}[\textbf{Proof of Theorem~\ref{thm:hp-interactive-fano}}]
Fix an algorithm \(\ALG\in\mathcal D\), a reference distribution
\(Q_\ALG\in\Delta(\mathcal X)\), and a level \(\epsilon\in[0,1]\). Define
\[
\rho_{\Delta,\ALG}
:=
\mathbb P_{M\sim\mu,\;X\sim \mathbb P^{M,\ALG}}
\bigl(L(M,X)\le \Delta\bigr),
\]
and recall that
\[
\bar\rho_{\Delta,Q_\ALG}
:=
\mathbb P_{M\sim\mu,\;X\sim Q_\ALG}
\bigl(L(M,X)\le \Delta\bigr).
\]
Consider the following two distributions on \(\mathcal M\times\mathcal X\):
\[
P_0:\quad M\sim\mu,\;X\sim\mathbb P^{M,\ALG},
\]
and
\[
P_1:\quad M\sim\mu,\;X\sim Q_\ALG.
\]
Let
\[
\psi(M,X):=\mathbf 1\{L(M,X)\le \Delta\}.
\]
Under \(P_0\), the random variable \(\psi(M,X)\) has distribution
\(\Bern(\rho_{\Delta,\ALG})\), while under \(P_1\), it has distribution
\(\Bern(\bar\rho_{\Delta,Q_\ALG})\). By the data-processing inequality for
\(f\)-divergences,
\[
D_f\!\left(
\Bern(\rho_{\Delta,\ALG})
\,\middle\|\,
\Bern(\bar\rho_{\Delta,Q_\ALG})
\right)
\le
D_f(P_0\|P_1).
\]
Since \(P_0\) and \(P_1\) have the same marginal distribution \(\mu\) on
\(\mathcal M\), we have
\[
D_f(P_0\|P_1)
=
\mathbb E_{M\sim\mu}
\left[
D_f\!\left(\mathbb P^{M,\ALG}\|Q_\ALG\right)
\right].
\]
Therefore,
\[
D_f\!\left(
\Bern(\rho_{\Delta,\ALG})
\,\middle\|\,
\Bern(\bar\rho_{\Delta,Q_\ALG})
\right)
\le
\mathbb E_{M\sim\mu}
\left[
D_f\!\left(\mathbb P^{M,\ALG}\|Q_\ALG\right)
\right].
\]

Suppose now that
\[
\mathbb E_{M\sim\mu}
\left[
D_f\!\left(\mathbb P^{M,\ALG}\|Q_\ALG\right)
\right]
<
d_{f,\epsilon}(\bar\rho_{\Delta,Q_\ALG}).
\]
Then
\[
D_f\!\left(
\Bern(\rho_{\Delta,\ALG})
\,\middle\|\,
\Bern(\bar\rho_{\Delta,Q_\ALG})
\right)
<
d_{f,\epsilon}(\bar\rho_{\Delta,Q_\ALG}).
\]
By the definition
\[
d_{f,\epsilon}(p)
=
\begin{cases}
D_f\!\left(\Bern(1-\epsilon)\middle\|
\Bern(p)\right), & p\le 1-\epsilon,\\
0, & p>1-\epsilon,
\end{cases}
\]
and the nonnegativity of \(f\)-divergences, the preceding strict inequality
implies
\[
\bar\rho_{\Delta,Q_\ALG}<1-\epsilon
\]
and
\[
D_f\!\left(
\Bern(\rho_{\Delta,\ALG})
\,\middle\|\,
\Bern(\bar\rho_{\Delta,Q_\ALG})
\right)
<
D_f\!\left(
\Bern(1-\epsilon)
\,\middle\|\,
\Bern(\bar\rho_{\Delta,Q_\ALG})
\right).
\]

We now consider two cases. First, if
\[
\rho_{\Delta,\ALG}\le \bar\rho_{\Delta,Q_\ALG},
\]
then, since \(\bar\rho_{\Delta,Q_\ALG}<1-\epsilon\), we immediately get
\[
\rho_{\Delta,\ALG}<1-\epsilon.
\]
Second, suppose that
\[
\rho_{\Delta,\ALG}> \bar\rho_{\Delta,Q_\ALG}.
\]
By the monotonicity property of binary \(f\)-divergences
\cite[Lemma~C.1]{chen2024assouad}, the map
\[
x\mapsto
D_f\!\left(\Bern(x)\middle\|
\Bern(\bar\rho_{\Delta,Q_\ALG})\right)
\]
is nondecreasing for \(x\ge \bar\rho_{\Delta,Q_\ALG}\). Hence the strict
inequality
\[
D_f\!\left(
\Bern(\rho_{\Delta,\ALG})
\,\middle\|\,
\Bern(\bar\rho_{\Delta,Q_\ALG})
\right)
<
D_f\!\left(
\Bern(1-\epsilon)
\,\middle\|\,
\Bern(\bar\rho_{\Delta,Q_\ALG})
\right)
\]
implies
\[
\rho_{\Delta,\ALG}<1-\epsilon.
\]
Thus, in both cases,
\[
\rho_{\Delta,\ALG}<1-\epsilon.
\]
Consequently,
\[
\mathbb P_{M\sim\mu,\;X\sim \mathbb P^{M,\ALG}}
\bigl(L(M,X)>\Delta\bigr)
=
1-\rho_{\Delta,\ALG}
>
\epsilon.
\]

We now make the argument uniform over algorithms. Let
\(\delta<\epsilon^\star\). By the definition of \(\epsilon^\star\), there
exist \(\epsilon\in(\delta,\epsilon^\star)\) and a reference family
\(\{Q_\ALG\}_{\ALG\in\mathcal D}\) such that
\[
\sup_{\ALG\in\mathcal D}
\left[
\mathbb E_{M\sim\mu}
\left[
D_f\!\left(\mathbb P^{M,\ALG}\|Q_\ALG\right)
\right]
-
d_{f,\epsilon}(\bar\rho_{\Delta,Q_\ALG})
\right]
<0.
\]
Equivalently, for every \(\ALG\in\mathcal D\),
\[
\mathbb E_{M\sim\mu}
\left[
D_f\!\left(\mathbb P^{M,\ALG}\|Q_\ALG\right)
\right]
<
d_{f,\epsilon}(\bar\rho_{\Delta,Q_\ALG}).
\]
Applying the preceding argument to every \(\ALG\in\mathcal D\) gives
\[
\inf_{\ALG\in\mathcal D}
\mathbb P_{M\sim\mu,\;X\sim \mathbb P^{M,\ALG}}
\bigl(L(M,X)>\Delta\bigr)
\ge
\epsilon.
\]
Since the probability on the left averages over \(M\sim\mu\), it is bounded
above by the corresponding worst-case probability over \(M\in\mathcal M\).
Therefore,
\[
\inf_{\ALG\in\mathcal D}
\sup_{M\in\mathcal M}
\mathbb P^{M,\ALG}
\bigl(L(M,X)>\Delta\bigr)
\ge
\epsilon
>
\delta.
\]
By Theorem~\ref{thm:lower-minimax-quantile-relation}, this implies
\[
\mathfrak M_-(\delta)\ge \Delta.
\]
This proves the theorem.
\end{proof}

\begin{proof}[\textbf{Proof of Theorem~\ref{thm:hp-lecam}}]
We prove part (a) first. Fix \(a\in(0,1)\), and define the thresholded loss
\[
g_a(t):=\mathbf 1\{t>a\Delta\}.
\]
The separation condition implies that for every \(x\in\mathcal X\),
\[
g_a(L(M_1,x))+g_a(L(M_2,x))\ge 1.
\]
Applying the interactive two-point Le Cam's inequality for bounded losses from \cite[Prop.~4]{chen2024assouad} to the pair \((M_1,M_2)\) and the losses \(g_a\circ L\), we obtain
\[
\inf_{\ALG\in\mathcal D}\sup_{j\in\{1,2\}}
\mathbb E^{M_j,\ALG}[g_a(L(M_j,X))]
\ge
\frac12\Bigl(1-\sup_{\ALG\in\mathcal D}\TV(\mathbb P^{M_1,\ALG},\mathbb P^{M_2,\ALG})\Bigr).
\]
Since
\[
\mathbb E^{M_j,\ALG}[g_a(L(M_j,X))]
=
\mathbb P^{M_j,\ALG}(L(M_j,X)>a\Delta),
\]
the assumption in part (a) yields
\[
\inf_{\ALG\in\mathcal D}\sup_{M\in\{M_1,M_2\}}
\mathbb P^{M,\ALG}(L(M,X)>a\Delta)
>
\delta.
\]
Using Theorem~\ref{thm:lower-minimax-quantile-relation}, we have
\[
\mathfrak M_{-}(\delta)\ge a\Delta.
\]
Letting \(a\uparrow 1\) implies
\[
\mathfrak M_{-}(\delta)\ge \Delta.
\]

For part (b), by the Bretagnolle--Huber inequality, for every \(\ALG\in\mathcal D\),
\[
\TV(P^{M_1,\ALG},P^{M_2,\ALG})
\le
\sqrt{1-\exp\!\left(
-\KL(P^{M_1,\ALG}\|P^{M_2,\ALG})
\right)}.
\]
This implies
\[
\sup_{\ALG\in\mathcal D}
\TV(P^{M_1,\ALG},P^{M_2,\ALG})
\le
\sup_{\ALG\in\mathcal D}
\sqrt{1-\exp\!\left(
-
\KL(P^{M_1,\ALG}\|P^{M_2,\ALG})
\right)}.
\]
Since \(t\mapsto \sqrt{1-e^{-t}}\) is nondecreasing on \([0,\infty)\), it follows that
\[
\sup_{\ALG\in\mathcal D}
\TV(P^{M_1,\ALG},P^{M_2,\ALG})
\le
\sqrt{1-\exp\!\left(
-\sup_{\ALG\in\mathcal D}
\KL(P^{M_1,\ALG}\|P^{M_2,\ALG})
\right)}.
\]
Under the assumed KL condition,
\[
\sup_{\ALG\in\mathcal D}
\TV(P^{M_1,\ALG},P^{M_2,\ALG})
<
\sqrt{1-4\delta(1-\delta)}
=
1-2\delta,
\]
where the last equality uses \(\delta\in(0,\tfrac12)\). Hence the condition in
part~(a) holds, and the conclusion follows.
\end{proof}

\begin{proof}[\textbf{Proof of Theorem~\ref{thm:gaussian-private-meta}}]
Fix any feasible private rule
\(\ALG^{\mathrm{GMI}}\in\mathcal D_\varepsilon^{\mathrm{GMI}}\). By
Assumption~\ref{ass:noise-floor},
\[
\sigma^2\ge \sigma_{\min}^2(\varepsilon,d).
\]
Now fix \(u>0\), and let \(M_1(u),M_2(u)\) be the two models from
Assumption~\ref{ass:two-point-gaussian}. The separation condition gives
\[
L(M_1(u),y)+L(M_2(u),y)\ge 2\phi(u)
\qquad
\text{for all }y\in\mathcal Y.
\]
Moreover,
\[
\KL\!\left(
P^{M_1(u),\ALG^{\mathrm{GMI}}}\middle\|P^{M_2(u),\ALG^{\mathrm{GMI}}}
\right)
\le
\frac{a u^2}{1+\sigma^2}
\le
\frac{a u^2}{1+\sigma_{\min}^2(\varepsilon,d)}.
\]
Hence, by Corollary~\ref{cor:private-hp-lecam}(b), applied to the restricted class
\(\mathcal D_\varepsilon^{\mathrm{GMI}}\),
\[
\mathfrak M^{\mathrm{GMI}}_-(\delta;\varepsilon)\ge \phi(u)
\]
whenever
\[
\frac{a u^2}{1+\sigma_{\min}^2(\varepsilon,d)}
<
\Lambda_\delta.
\]
For arbitrary \(\eta\in(0,1)\), choose
\[
u_\eta
=
(1-\eta)
\sqrt{
\frac{1+\sigma_{\min}^2(\varepsilon,d)}{a}\,
\Lambda_\delta
}.
\]
Then the strict KL condition holds, and therefore
\[
\mathfrak M^{\mathrm{GMI}}_-(\delta;\varepsilon)
\ge
\phi(u_\eta).
\]
Letting \(\eta\downarrow0\) and using the monotonicity and left-continuity of \(\phi\) gives
\[
\mathfrak M^{\mathrm{GMI}}_-(\delta;\varepsilon)
\ge
\phi\!\left(
\sqrt{
\frac{1+\sigma_{\min}^2(\varepsilon,d)}{a}\,
\Lambda_\delta
}
\right).
\]
The lower bound for \(\mathfrak M^{\mathrm{GMI}}(\delta;\varepsilon)\) follows from
Corollary~\ref{cor:private-lower-minimax-quantile-relation}, applied to the privacy-constrained
class \(\mathcal D_\varepsilon^{\mathrm{GMI}}\).
\end{proof}

\begin{proof}[\textbf{Proof of Theorem~\ref{thm:nonprivate-mean}}]
We apply Theorem~\ref{thm:hp-lecam}. Fix \(u>0\), and consider the hard pair
\[
M_1:\theta=u,
\qquad
M_2:\theta=-u.
\]
We take the loss to be squared-error:
\[
L(\theta,x)=(\hat\theta-\theta)^2.
\]
For any outcome \(x=(x_{1:n},\hat\theta)\),
\[
L(M_1,x)+L(M_2,x)
=
(\hat\theta-u)^2+(\hat\theta+u)^2
=
2\hat\theta^2+2u^2
\ge 2u^2.
\]
Thus the separation condition holds with
\[
\Delta=u^2.
\]

Let \(P_i=\mathbb P^{M_i,\ALG}\) denote the law of
\((X_{1:n},\hat\theta)\) under \(M_i\). By data processing,
\[
\KL(P_1\|P_2)
\le
\KL\!\Bigl(
\mathcal N(u,1)^{\otimes n}
\Big\|
\mathcal N(-u,1)^{\otimes n}
\Bigr).
\]
The KL divergence between the two product Gaussian laws is
\(
2nu^2.
\)
Let
\[
\Lambda_\delta
:=
\log\!\Bigl(\frac{1}{4\delta(1-\delta)}\Bigr).
\]
Whenever
\[
2nu^2<\Lambda_\delta,
\]
Theorem~\ref{thm:hp-lecam}(b) yields
\[
\mathfrak M_-(\delta)\ge u^2.
\]
Choosing
\[
u=(1-\eta)\sqrt{\frac{\Lambda_\delta}{2n}}
\]
for arbitrary \(\eta\in(0,1)\), and then letting \(\eta\downarrow0\), gives
\[
\mathfrak M_-(\delta)
\ge
\frac{1}{2n}
\log\!\Bigl(\frac{1}{4\delta(1-\delta)}\Bigr).
\]
The lower bound for \(\mathfrak M(\delta)\) follows from
Theorem~\ref{thm:lower-minimax-quantile-relation}.
\end{proof}

\begin{proof}[\textbf{Proof of Theorem~\ref{thm:nonprivate-two-arm}}]
We apply Theorem~\ref{thm:hp-lecam}. Fix \(g\in(0,1]\), and consider the
bounded symmetric hard pair
\[
M_1:\quad
(\mu_1,\mu_2)
=
\Bigl(\frac12+\frac g2,\frac12-\frac g2\Bigr),
\]
\[
M_2:\quad
(\mu_1,\mu_2)
=
\Bigl(\frac12-\frac g2,\frac12+\frac g2\Bigr).
\]
Since \(g\le 1\), both models belong to \(\mathcal M\).

For any transcript \(x\), let \(N_a(x)\) denote the number of pulls of arm
\(a\). Under \(M_1\), arm \(1\) is optimal and the gap is \(g\), so
\[
L(M_1,x)=g\,N_2(x).
\]
Under \(M_2\), arm \(2\) is optimal and the gap is again \(g\), so
\[
L(M_2,x)=g\,N_1(x).
\]
Hence, for every transcript \(x\),
\[
L(M_1,x)+L(M_2,x)
=
g\bigl(N_1(x)+N_2(x)\bigr)
=
gT.
\]
Thus the separation condition in Theorem~\ref{thm:hp-lecam} holds with
\[
\Delta=\frac{gT}{2}.
\]

Let \(P_i=\mathbb P^{M_i,\ALG}\) denote the transcript law under \(M_i\).
By the arm-wise KL decomposition for adaptive bandits \cite[Lemma 15.1]{lattimore2020bandit},
\[
\KL(P_1\|P_2)
=
\sum_{a=1}^2
\mathbb E_{M_1}[N_a(T)]\,
\KL\!\Bigl(
\mathcal N(\mu_a^{(1)},1)
\Big\|
\mathcal N(\mu_a^{(2)},1)
\Bigr).
\]
For each arm \(a\), the two means differ by \(g\). Therefore,
\[
\KL\!\Bigl(
\mathcal N(\mu_a^{(1)},1)
\Big\|
\mathcal N(\mu_a^{(2)},1)
\Bigr)
=
\frac{g^2}{2}.
\]
Since \(N_1(T)+N_2(T)=T\), we obtain
\[
\KL(P_1\|P_2)
=
\frac{g^2T}{2}.
\]

Let
\[
\Lambda_\delta
:=
\log\!\Bigl(\frac{1}{4\delta(1-\delta)}\Bigr).
\]
Whenever
\[
\frac{g^2T}{2}<\Lambda_\delta,
\]
Theorem~\ref{thm:hp-lecam}(b) yields
\[
\mathfrak M_-(\delta)\ge \frac{gT}{2}.
\]

We now choose \(g\) as large as permitted by both the KL constraint and the
bounded-mean constraint \(g\le 1\). For arbitrary \(\eta\in(0,1)\), set
\[
g
=
(1-\eta)
\min\!\left\{
1,\,
\sqrt{\frac{2\Lambda_\delta}{T}}
\right\}.
\]
Then \(g\in(0,1]\), so \(M_1,M_2\in\mathcal M\), and the strict KL condition
holds. Hence
\[
\mathfrak M_-(\delta)
\ge
\frac{(1-\eta)T}{2}
\min\!\left\{
1,\,
\sqrt{\frac{2\Lambda_\delta}{T}}
\right\}.
\]
Letting \(\eta\downarrow0\) gives
\[
\mathfrak M_-(\delta)
\ge
\frac{T}{2}
\min\!\left\{
1,\,
\sqrt{\frac{2\Lambda_\delta}{T}}
\right\}.
\]
The lower bound for \(\mathfrak M(\delta)\) follows from
Theorem~\ref{thm:lower-minimax-quantile-relation}.
\end{proof}

\begin{proof}[\textbf{Proof of Theorem \ref{thm:karmed-fano-bandit}}]
Fix \(K\ge2\), \(T\ge2\), and \(\delta\in(0,1-1/K)\). We verify the
conditions of Theorem~\ref{thm:hp-interactive-fano} for a suitable finite
subfamily of \(K\)-armed Gaussian bandits.

For a gap parameter \(\alpha\in(0,1]\), define the reference model \(M_0\) by
\[
\mu_a^{(0)}=0,
\qquad a=1,\dots,K.
\]
For each \(j\in[K]\), define the alternative model \(M_j\) by
\[
\mu_j^{(j)}=\alpha,
\qquad
\mu_a^{(j)}=0
\quad\text{for }a\neq j.
\]
Since \(\alpha\in(0,1]\), all these models belong to the class
\([0,1]^K\). Under \(M_j\), arm \(j\) is the unique optimal arm and all other
arms have gap \(\alpha\).

We use the uniform prior on the hard alternatives:
\[
\mu(M_j)=\frac1K,
\qquad j=1,\dots,K.
\]
The outcome space is the transcript space
\[
\mathcal X=([K]\times\mathbb R)^T,
\]
and the outcome is
\[
X=H_T=(A_1,R_1,\dots,A_T,R_T).
\]
For a policy \(\ALG\), let \(P_j^\ALG\) denote the law of \(H_T\) under model
\(M_j\) and policy \(\ALG\), and let \(P_0^\ALG\) denote the law of \(H_T\)
under the reference model \(M_0\).

Define
\[
N_a(T):=\sum_{t=1}^T \mathbf 1\{A_t=a\}
\]
to be the number of pulls of arm \(a\). Under \(M_j\), the pseudo-regret is
\[
L(M_j,H_T)
=
\alpha\bigl(T-N_j(T)\bigr).
\]
Set the threshold
\[
\Delta_\alpha:=\frac{\alpha(T-1)}{2}.
\]
If
\[
L(M_j,H_T)\le \Delta_\alpha,
\]
then
\[
\alpha(T-N_j(T))
\le
\frac{\alpha(T-1)}{2}.
\]
Since \(\alpha>0\), this implies
\[
N_j(T)\ge \frac{T+1}{2}>\frac T2.
\]
Therefore,
\[
\{L(M_j,H_T)\le \Delta_\alpha\}
\subseteq
\left\{N_j(T)>\frac T2\right\}.
\]

We now choose the \(f\)-divergence. Let
\[
f(t)=-\log t.
\]
Then
\[
D_f(P\|Q)=\KL(Q\|P),
\]
so \(D_f\) is reverse KL. For each algorithm \(\ALG\), choose the
algorithm-dependent reference law
\[
Q_\ALG:=P_0^\ALG.
\]
This is admissible in Theorem~\ref{thm:hp-interactive-fano}, since the theorem
allows a reference family \(\{Q_\ALG\}_{\ALG\in\mathcal D}\).

We first verify the reference-success bound. For this reference family,
\[
\bar\rho_{\Delta_\alpha,Q_\ALG}
=
\mathbb P_{M\sim\mu,\;H_T\sim P_0^\ALG}
\bigl(L(M,H_T)\le \Delta_\alpha\bigr).
\]
Using the uniform prior,
\[
\bar\rho_{\Delta_\alpha,Q_\ALG}
=
\frac1K
\sum_{j=1}^K
P_0^\ALG
\bigl(L(M_j,H_T)\le \Delta_\alpha\bigr).
\]
By the implication above,
\[
P_0^\ALG
\bigl(L(M_j,H_T)\le \Delta_\alpha\bigr)
\le
P_0^\ALG
\left(N_j(T)>\frac T2\right).
\]
The events
\[
\left\{N_j(T)>\frac T2\right\},
\qquad j=1,\dots,K,
\]
are disjoint, because two distinct arms cannot both be pulled more than half
of the \(T\) rounds. Hence
\[
\sum_{j=1}^K
P_0^\ALG
\left(N_j(T)>\frac T2\right)
\le 1.
\]
Therefore,
\[
\begin{aligned}
\bar\rho_{\Delta_\alpha,Q_\ALG}
&\le
\frac1K
\sum_{j=1}^K
P_0^\ALG
\left(N_j(T)>\frac T2\right)\\
&\le
\frac1K.
\end{aligned}
\]

Next we verify the average divergence bound. Since \(D_f(P\|Q)=\KL(Q\|P)\)
and \(Q_\ALG=P_0^\ALG\), we have
\[
\mathbb E_{M\sim\mu}
D_f(P^{M,\ALG}\|Q_\ALG)
=
\frac1K
\sum_{j=1}^K
D_f(P_j^\ALG\|P_0^\ALG)
=
\frac1K
\sum_{j=1}^K
\KL(P_0^\ALG\|P_j^\ALG).
\]

We compute each term using the arm-wise KL decomposition for
adaptive bandits \cite[Lemma 15.1]{lattimore2020bandit}.
Grouping the terms according to the arm pulled,
\[
\KL(P_0^\ALG\|P_j^\ALG)
=
\sum_{a=1}^K
\mathbb E_0^\ALG[N_a(T)]
\KL\!\left(
\mathcal N(\mu_a^{(0)},1)
\,\middle\|\,
\mathcal N(\mu_a^{(j)},1)
\right).
\]
By construction,
\[
\mu_a^{(0)}=0
\quad\text{for all }a,
\]
and
\[
\mu_j^{(j)}=\alpha,
\qquad
\mu_a^{(j)}=0
\quad\text{for }a\neq j.
\]
Hence all terms in the arm-wise KL sum vanish except the term \(a=j\). Thus
\[
\KL(P_0^\ALG\|P_j^\ALG)
=
\mathbb E_0^\ALG[N_j(T)]
\KL\!\left(
\mathcal N(0,1)
\,\middle\|\,
\mathcal N(\alpha,1)
\right).
\]
Since
\[
\KL\!\left(
\mathcal N(0,1)
\,\middle\|\,
\mathcal N(\alpha,1)
\right)
=
\frac{\alpha^2}{2},
\]
we obtain
\[
\KL(P_0^\ALG\|P_j^\ALG)
=
\frac{\alpha^2}{2}
\mathbb E_0^\ALG[N_j(T)].
\]

Averaging over \(j\), we get
\[
\begin{aligned}
\mathbb E_{M\sim\mu}
D_f(P^{M,\ALG}\|Q_\ALG)
&=
\frac1K
\sum_{j=1}^K
\KL(P_0^\ALG\|P_j^\ALG)\\
&=
\frac{\alpha^2}{2K}
\sum_{j=1}^K
\mathbb E_0^\ALG[N_j(T)].
\end{aligned}
\]
Since exactly one arm is pulled at each round,
\[
\sum_{j=1}^K N_j(T)=T.
\]
Taking expectation under \(M_0\), we obtain
\[
\sum_{j=1}^K
\mathbb E_0^\ALG[N_j(T)]
=
T.
\]
Therefore,
\[
\mathbb E_{M\sim\mu}
D_f(P^{M,\ALG}\|Q_\ALG)
=
\frac{\alpha^2T}{2K}.
\]
This identity holds for every \(\ALG\), and therefore
\[
\sup_{\ALG\in\mathcal D}
\mathbb E_{M\sim\mu}
D_f(P^{M,\ALG}\|Q_\ALG)
=
\frac{\alpha^2T}{2K}.
\]

We now verify the binary Fano term. For reverse KL,
\[
d_{f,\epsilon}(p)
=
D_f\bigl(\Bern(1-\epsilon)\,\|\,\Bern(p)\bigr)
=
\KL\bigl(\Bern(p)\,\|\,\Bern(1-\epsilon)\bigr)
=
\KL(p\|1-\epsilon),
\]
whenever \(p\le 1-\epsilon\). If \(\epsilon<1-1/K\), then
\[
\frac1K<1-\epsilon.
\]
Since
\[
\bar\rho_{\Delta_\alpha,Q_\ALG}\le \frac1K,
\]
we are in the first branch of the definition of \(d_{f,\epsilon}\), and
\[
d_{f,\epsilon}(\bar\rho_{\Delta_\alpha,Q_\ALG})
=
\KL(\bar\rho_{\Delta_\alpha,Q_\ALG}\|1-\epsilon).
\]
The map
\[
p\mapsto \KL(p\|1-\epsilon)
\]
is decreasing on \(p\in[0,1-\epsilon]\). Indeed,
\[
\frac{\partial}{\partial p}\KL(p\|q)
=
\log\left(\frac{p(1-q)}{q(1-p)}\right),
\]
which is negative whenever \(p<q\). Therefore,
\[
\bar\rho_{\Delta_\alpha,Q_\ALG}\le \frac1K<1-\epsilon
\]
implies
\[
d_{f,\epsilon}(\bar\rho_{\Delta_\alpha,Q_\ALG})
\ge
\KL\left(\frac1K\,\middle\|\,1-\epsilon\right).
\]

Combining the divergence bound and the binary Fano bound, we get, for every
\(\ALG\in\mathcal D\),
\[
\mathbb E_{M\sim\mu}
D_f(P^{M,\ALG}\|Q_\ALG)
-
d_{f,\epsilon}(\bar\rho_{\Delta_\alpha,Q_\ALG})
\le
\frac{\alpha^2T}{2K}
-
\KL\left(\frac1K\,\middle\|\,1-\epsilon\right).
\]
Thus the condition in Theorem~\ref{thm:hp-interactive-fano} is satisfied
whenever
\[
\frac{\alpha^2T}{2K}
<
\KL\left(\frac1K\,\middle\|\,1-\epsilon\right).
\]

Now fix \(\epsilon\in(0,1-1/K)\). If
\[
\alpha
<
\min\left\{
1,
\sqrt{
\frac{2K}{T}
\KL\left(\frac1K\,\middle\|\,1-\epsilon\right)
}
\right\},
\]
then all assumptions and the strict Fano condition are satisfied. Hence, by
Theorem~\ref{thm:hp-interactive-fano},
\[
\mathfrak M_-(\delta')\ge \Delta_\alpha
=
\frac{\alpha(T-1)}{2}
\qquad
\text{for every } \delta'<\epsilon.
\]

At this point we have proved the following intermediate statement. For every
\(\epsilon\in(0,1-1/K)\) and every \(\alpha\in(0,1]\) such that
\[
\frac{\alpha^2T}{2K}
<
\KL\left(\frac1K\,\middle\|\,1-\epsilon\right),
\]
Theorem~\ref{thm:hp-interactive-fano} implies
\[
\mathfrak M_-(\delta')\ge \Delta_\alpha
=
\frac{\alpha(T-1)}{2}
\qquad
\text{for every } \delta'\in[0,\epsilon).
\]
It remains to choose the largest admissible value of \(\alpha\), up to an
arbitrarily small slack, and then remove the slack between \(\epsilon\) and
the target level \(\delta\).

Fix the target confidence level
\[
\delta\in(0,1-1/K).
\]
Choose any
\[
\epsilon\in(\delta,1-1/K).
\]
Define
\[
A_\epsilon
:=
\min\left\{
1,
\sqrt{
\frac{2K}{T}
\KL\left(\frac1K\,\middle\|\,1-\epsilon\right)
}
\right\}.
\]
For any \(\eta\in(0,1)\), set
\[
\alpha_{\epsilon,\eta}:=(1-\eta)A_\epsilon.
\]
Then
\[
0<\alpha_{\epsilon,\eta}\le 1.
\]
Moreover, since \((1-\eta)^2<1\),
\[
\alpha_{\epsilon,\eta}^2
=
(1-\eta)^2 A_\epsilon^2
\le
(1-\eta)^2
\frac{2K}{T}
\KL\left(\frac1K\,\middle\|\,1-\epsilon\right)
<
\frac{2K}{T}
\KL\left(\frac1K\,\middle\|\,1-\epsilon\right).
\]
Equivalently,
\[
\frac{\alpha_{\epsilon,\eta}^2T}{2K}
<
\KL\left(\frac1K\,\middle\|\,1-\epsilon\right).
\]
Thus, \(\alpha_{\epsilon,\eta}\) is an admissible gap parameter for the
intermediate statement above.

Therefore, for every \(\delta'\in[0,\epsilon)\),
\[
\mathfrak M_-(\delta')
\ge
\frac{T-1}{2}\alpha_{\epsilon,\eta}
=
\frac{T-1}{2}(1-\eta)A_\epsilon.
\]
Since the target level \(\delta\) belongs to \([0,\epsilon)\), we now set
\(\delta'=\delta\), which gives
\[
\mathfrak M_-(\delta)
\ge
\frac{T-1}{2}(1-\eta)A_\epsilon.
\]
This inequality holds for every \(\eta\in(0,1)\). Letting
\(\eta\downarrow0\) gives
\[
\mathfrak M_-(\delta)
\ge
\frac{T-1}{2}A_\epsilon.
\]
Thus, we have shown that for every
\[
\delta\in(0,1-1/K)
\]
and every
\[
\epsilon\in(\delta,1-1/K),
\]
one has
\[
\mathfrak M_-(\delta)
\ge
\frac{T-1}{2}
\min\left\{
1,
\sqrt{
\frac{2K}{T}
\KL\left(\frac1K\,\middle\|\,1-\epsilon\right)
}
\right\}.
\]

It remains to remove the auxiliary level \(\epsilon\). Fix
\[
\delta\in(0,1-1/K).
\]
From the preceding argument, for every
\[
\epsilon\in(\delta,1-1/K),
\]
we have
\[
\mathfrak M_-(\delta)
\ge
\frac{T-1}{2}
\min\left\{
1,
\sqrt{
\frac{2K}{T}
\KL\left(\frac1K\,\middle\|\,1-\epsilon\right)
}
\right\}.
\]
For this fixed \(\delta\), the left-hand side no longer depends on
\(\epsilon\). Only the right-hand side depends on \(\epsilon\). Therefore, since
the inequality holds for every \(\epsilon>\delta\) sufficiently close to
\(\delta\), we may let \(\epsilon\downarrow\delta\) on the right-hand side.

We now justify this limiting step. Since
\[
\delta\in(0,1-1/K),
\]
we have
\[
1-\delta\in(1/K,1)\subset(0,1).
\]
For fixed \(p=1/K\), the binary relative entropy
\[
q\mapsto \KL(p\|q)
=
p\log\frac{p}{q}
+
(1-p)\log\frac{1-p}{1-q}
\]
is continuous for \(q\in(0,1)\), because it is a sum of logarithmic functions
whose arguments are positive on this interval. Hence, as
\(\epsilon\downarrow\delta\),
\[
1-\epsilon\to 1-\delta
\]
and therefore
\[
\KL\left(\frac1K\,\middle\|\,1-\epsilon\right)
\to
\KL\left(\frac1K\,\middle\|\,1-\delta\right).
\]
Moreover, the function
\[
x\mapsto
\min\left\{
1,
\sqrt{\frac{2K}{T}x}
\right\}
\]
is continuous for \(x\ge0\). Thus the right-hand side converges to
\[
\frac{T-1}{2}
\min\left\{
1,
\sqrt{
\frac{2K}{T}
\KL\left(\frac1K\,\middle\|\,1-\delta\right)
}
\right\}.
\]
Consequently,
\[
\mathfrak M_-(\delta)
\ge
\frac{T-1}{2}
\min\left\{
1,
\sqrt{
\frac{2K}{T}
\KL\left(\frac1K\,\middle\|\,1-\delta\right)
}
\right\}.
\]
This proves the theorem.
\end{proof}

\begin{proof}[\textbf{Proof of Theorem~\ref{thm:private-mean}}]
We verify the hypotheses of Theorem~\ref{thm:gaussian-private-meta} with
\(d=n\), using squared-error loss.

First, since \(\hat\theta\) is a measurable function of
\(\widetilde Y_{1:n}\),
\[
I(S;Y)
=
I(X_{1:n};\widetilde Y_{1:n},\hat\theta)
=
I(X_{1:n};\widetilde Y_{1:n}).
\]
Because the coordinates are independent and each passes through a scalar
Gaussian channel,
\[
I(X_{1:n};\widetilde Y_{1:n})
=
\sum_{i=1}^n I(X_i;\widetilde Y_i)
=
\frac n2\log\!\Bigl(1+\frac{1}{\sigma^2}\Bigr).
\]
The privacy constraint \(I(S;Y)\le\varepsilon\) therefore implies
\[
\sigma^2\ge \frac{1}{e^{2\varepsilon/n}-1},
\]
so
\[
\sigma_{\min}^2(\varepsilon,n)
=
\frac{1}{e^{2\varepsilon/n}-1}.
\]

Next, fix \(u>0\), and consider the hard pair
\[
M_1(u):\theta=u,
\qquad
M_2(u):\theta=-u.
\]
For every released output \(y=(\widetilde y_{1:n},\hat\theta)\),
\[
L(M_1(u),y)+L(M_2(u),y)=
(\hat\theta-u)^2+(\hat\theta+u)^2
=
2\hat\theta^2+2u^2
\ge 2u^2.
\]
Thus the separation condition holds with \(\Delta=u^2\).

Finally, by data processing,
\[
\KL(P^{M_1(u),\ALG^{\mathrm{GMI}}}\|P^{M_2(u),\ALG^{\mathrm{GMI}}})
\le
\KL\!\Bigl(
\mathcal N(u,1+\sigma^2)^{\otimes n}
\Big\|
\mathcal N(-u,1+\sigma^2)^{\otimes n}
\Bigr)
=
\frac{2nu^2}{1+\sigma^2}.
\]
Since every feasible private rule satisfies
\[
\sigma^2\ge \sigma_{\min}^2(\varepsilon,n),
\]
we have
\[
\KL(P^{M_1(u),\ALG^{\mathrm{GMI}}}\|P^{M_2(u),\ALG^{\mathrm{GMI}}})
\le
\frac{2nu^2}{1+\sigma_{\min}^2(\varepsilon,n)}.
\]
Let
\[
\Lambda_\delta
:=
\log\!\Bigl(\frac{1}{4\delta(1-\delta)}\Bigr).
\]
Thus Assumption~\ref{ass:two-point-gaussian} holds with
\[
\phi(u)=u^2,
\qquad
a=2n.
\]
Applying Theorem~\ref{thm:gaussian-private-meta} gives 

\[
\mathfrak M^{\mathrm{GMI}}_-(\delta;\varepsilon)
\ge
\frac{1+\sigma_{\min}^2(\varepsilon,n)}{2n}
\log\!\Bigl(\frac{1}{4\delta(1-\delta)}\Bigr).
\]
The lower bound for \(\mathfrak M^{\mathrm{GMI}}(\delta;\varepsilon)\) follows from
Corollary~\ref{cor:private-lower-minimax-quantile-relation}.
\end{proof}

\begin{proof}[\textbf{Proof of Theorem~\ref{thm:private-bandit}}]
We first derive the privacy-induced noise floor. Let
\[
H_{t-1}=(A_1,\widetilde Y_1,\dots,A_{t-1},\widetilde Y_{t-1})
\]
denote the privatized history available to the policy before round \(t\). The
private policy is adapted to this filtration, so that
\[
A_t\sim \pi_t(\cdot\mid H_{t-1}),
\]
possibly with additional internal randomization independent of the reward and
privacy noises.

Recall that the sensitive object is the raw reward sequence
\[
S=(R_1,\dots,R_T),
\]
whereas the released transcript is
\[
Y=(A_1,\widetilde Y_1,\dots,A_T,\widetilde Y_T).
\]
By the chain rule for mutual information,
\[
I(S;Y)
=
\sum_{t=1}^T
I(R_{1:T};A_t,\widetilde Y_t\mid H_{t-1}).
\]
Moreover,
\begin{align*}
I(R_{1:T};A_t,\widetilde Y_t\mid H_{t-1})
&=
I(R_{1:T};A_t\mid H_{t-1})
+
I(R_{1:T};\widetilde Y_t\mid H_{t-1},A_t)  \\
&\ge
I(R_{1:T};\widetilde Y_t\mid H_{t-1},A_t)  \\
&\ge
I(R_t;\widetilde Y_t\mid H_{t-1},A_t),
\end{align*}
where the last inequality follows because \(R_t\) is a component of
\(R_{1:T}\).

Conditionally on \(H_{t-1}\) and \(A_t=a\), the raw reward satisfies
\[
R_t\sim \mathcal N(\mu_a,1),
\]
and the released observation is
\[
\widetilde Y_t=R_t+Z_t,
\qquad
Z_t\sim\mathcal N(0,\sigma^2),
\]
with \(Z_t\) independent of \(R_t\). Therefore,
\[
I(R_t;\widetilde Y_t\mid H_{t-1},A_t)
=
\frac12\log\left(1+\frac1{\sigma^2}\right).
\]
Hence
\[
I(S;Y)
\ge
\frac T2\log\left(1+\frac1{\sigma^2}\right).
\]
Since every feasible private rule satisfies \(I(S;Y)\le \varepsilon\), we must
have
\[
\sigma^2
\ge
\frac{1}{e^{2\varepsilon/T}-1}
=
\sigma_{\min}^2(\varepsilon,T).
\]

We now prove the lower bound. Fix \(g\in(0,1]\), and consider the bounded
symmetric hard pair
\[
M_1:\quad
(\mu_1,\mu_2)
=
\Bigl(\frac12+\frac g2,\frac12-\frac g2\Bigr),
\]
\[
M_2:\quad
(\mu_1,\mu_2)
=
\Bigl(\frac12-\frac g2,\frac12+\frac g2\Bigr).
\]
Since \(g\le 1\), both models belong to \(\mathcal M\).

For any released transcript \(y\), let \(N_a(y)\) denote the number of pulls of
arm \(a\). Under \(M_1\), arm \(1\) is optimal and the gap is \(g\), so
\[
L(M_1,y)=g\,N_2(y).
\]
Under \(M_2\), arm \(2\) is optimal and the gap is \(g\), so
\[
L(M_2,y)=g\,N_1(y).
\]
Thus, for every released transcript \(y\),
\[
L(M_1,y)+L(M_2,y)
=
g\bigl(N_1(y)+N_2(y)\bigr)
=
gT.
\]
Hence the separation condition in Corollary~\ref{cor:private-hp-lecam} holds with
\[
\Delta=\frac{gT}{2}.
\]

Let \(P_i=\mathbb P^{M_i,\ALG^{\mathrm{GMI}}}\) denote the released-transcript law under
\(M_i\) for a feasible private rule \(\ALG^{\mathrm{GMI}}\). By the adaptive bandit KL
decomposition \cite[Lemma 15.1]{lattimore2020bandit},
\[
\KL(P_1\|P_2)
=
\sum_{a=1}^2
\mathbb E_{M_1,\ALG^{\mathrm{GMI}}}[N_a(T)]\,
\KL(\nu_a^{(1)}\|\nu_a^{(2)}),
\]
where
\[
\nu_a^{(i)}
=
\mathcal N(\mu_a^{(i)},1+\sigma^2)
\]
is the law of the privatized reward from arm \(a\) under \(M_i\). For each arm,
the two means differ by \(g\), and therefore
\[
\KL(\nu_a^{(1)}\|\nu_a^{(2)})
=
\frac{g^2}{2(1+\sigma^2)}.
\]
Since \(N_1(T)+N_2(T)=T\),
\[
\KL(P_1\|P_2)
=
\frac{g^2T}{2(1+\sigma^2)}.
\]
Using the privacy-induced noise floor
\[
\sigma^2\ge\sigma_{\min}^2(\varepsilon,T),
\]
we obtain
\[
\KL(P_1\|P_2)
\le
\frac{g^2T}{2(1+\sigma_{\min}^2(\varepsilon,T))}.
\]

Whenever
\[
\frac{g^2T}{2(1+\sigma_{\min}^2(\varepsilon,T))}
<
\Lambda_\delta,
\]
Corollary~\ref{cor:private-hp-lecam}(b), applied to the private decision
class, yields
\[
\mathfrak M^{\mathrm{GMI}}_-(\delta;\varepsilon)\ge \frac{gT}{2}.
\]

Choose, for arbitrary \(\eta\in(0,1)\),
\[
g
=
(1-\eta)
\min\!\left\{
1,\,
\sqrt{
\frac{2(1+\sigma_{\min}^2(\varepsilon,T))\Lambda_\delta}{T}
}
\right\}.
\]
Then \(g\in(0,1]\), so the hard pair belongs to the bounded model class, and
the strict KL condition holds. Hence
\[
\mathfrak M^{\mathrm{GMI}}_-(\delta;\varepsilon)
\ge
\frac{(1-\eta)T}{2}
\min\!\left\{
1,\,
\sqrt{
\frac{2(1+\sigma_{\min}^2(\varepsilon,T))\Lambda_\delta}{T}
}
\right\}.
\]
Letting \(\eta\downarrow0\) gives the claimed bound on
\(\mathfrak M^{\mathrm{GMI}}_-(\delta;\varepsilon)\). The lower bound for
\(\mathfrak M^{\mathrm{GMI}}(\delta;\varepsilon)\) follows from
Corollary~\ref{cor:private-lower-minimax-quantile-relation}.
\end{proof}

\end{document}